\documentclass{article}

\PassOptionsToPackage{numbers, sort&compress}{natbib}
\usepackage[preprint]{neurips_2020}

\usepackage[utf8]{inputenc} 
\usepackage[T1]{fontenc}    
\usepackage{hyperref}       
\usepackage{url}            
\usepackage{booktabs}       
\usepackage{amsfonts}       
\usepackage{nicefrac}       
\usepackage{microtype}      
\usepackage{multirow}
\usepackage{subfigure}
\usepackage{xcolor}
\usepackage{amsmath,amsthm,amssymb,bm}
\usepackage{mathtools}
\usepackage{tikz}
\usepackage{tkz-graph}
\usetikzlibrary {arrows,automata,positioning}
\newtheorem{theorem}{Theorem}[section]
\usepackage{dsfont}
\usepackage{wasysym}%

\usepackage[cal=dutchcal,
calscaled=1,
bb=boondox,
scr=euler]{mathalfa}
\usepackage{wrapfig}
\usepackage{verbatim}
\usepackage[capitalise]{cleveref}

\newtheorem{lemma}{Lemma}
\newtheorem{definition}{Definition}

\newtheorem{claim}{Claim}
\newtheorem{example}{Example}

\usepackage[framemethod=TikZ]{mdframed}
\mdfdefinestyle{MyFrame}{%
    linecolor=black,
    outerlinewidth=.3pt,
    roundcorner=5pt,
    innertopmargin=1pt, 
    innerbottommargin=1pt, 
    innerrightmargin=1pt,
    innerleftmargin=1pt,
    backgroundcolor=black!0!white}
    
\mdfdefinestyle{MyFrame2}{%
    linecolor=white,
    outerlinewidth=1pt,
    roundcorner=2pt,
    innertopmargin=\baselineskip,
    innerbottommargin=\baselineskip,
    innerrightmargin=5pt,
    innerleftmargin=5pt,
    backgroundcolor=black!3!white}

\newcommand{\ie}[0]{\emph{i.e.},~}
\newcommand{\eg}[0]{\emph{e.g.},~}

\newcommand{\defeq}{\ensuremath{\doteq}}
\newcommand{\diag}{\operatorname{diag}}
\newcommand{\pool}{\operatorname{pool}}

\newcommand{\kw}[1]{{\small\textsc{\MakeLowercase{#1}}}}
\newcommand{\mat}[1]{\ensuremath{\MakeUppercase{\mathbf{#1}}}}
\newcommand{\vect}[1]{\ensuremath{\MakeLowercase{\mathbf{#1}}}}
\newcommand{\set}[1]{\ensuremath{\mathbb{#1}}}
\newcommand{\multiset}[1]{\ensuremath{\widetilde{\set{#1}}}}
\newcommand{\gr}[1]{\ensuremath{\mathcal{#1}}}
\newcommand{\prm}[1]{\ensuremath{^{#1}}}
\newcommand{\grn}[2]{\ensuremath{\gr{#1}\prm{#2}}}
\newcommand{\tuple}[1]{\ensuremath{\langle{#1} \rangle}}

\newcommand{\Real}{\mathds{R}}
\renewcommand{\Re}{\mathds{R}}
\newcommand{\Nat}{\mathds{N}} 
\newcommand{\eye}{\mat{I}}
\newcommand{\ones}{\vect{1}}

\newcommand{\G}{\mat{G}}
\newcommand{\W}{\mat{W}}
\newcommand{\B}{\mat{B}}
\newcommand{\V}{\mat{V}}
\newcommand{\g}{\gr{g}}
\renewcommand{\H}{\mat{H}}

\newcommand{\X}{\mat{X}}
\newcommand{\Y}{\mat{Y}}
\newcommand{\Z}{\mat{Z}}
\newcommand{\x}{\vect{x}}

\newcommand{\K}{\mat{K}}
\newcommand{\w}{\vect{w}}
\renewcommand{\b}{\vect{b}}

\newcommand{\RR}{\set{R}}
\newcommand{\RRR}{\mathfrak{R}}
\newcommand{\XX}{\set{X}}
\newcommand{\BB}{\set{B}}
\newcommand{\WW}{\set{W}}

\newcommand{\n}{\vect{n}}
\newcommand{\m}{\vect{m}}
\newcommand{\st}{\;|\;}

\renewcommand{\vec}[1]{\ensuremath{\operatorname{vec}({#1})}}
\newcommand{\unvec}[1]{\ensuremath{\operatorname{vec}^{-1}({#1})}}

\newbool{APX}
\booltrue{APX}

\title{Equivariant Entity Relationship Networks}

\author{%
  Devon Graham \\
  University of British Columbia\\
  \texttt{drgraham@cs.ubc.ca} 
   \And
    Junhao Wang \\
  McGill University \& Mila\\
  \texttt{junhao.wang@mail.mcgill.ca} 
\And
  Siamak Ravanbakhsh \\
  McGill University \& Mila\\
  \texttt{siamak@cs.mcgill.ca} 
}

\begin{document}

\maketitle

\begin{abstract}
The relational model is a ubiquitous representation of big-data, in part due to its extensive use in databases. In this paper, we propose the Equivariant Entity-Relationship Network (EERN), which is a Multilayer Perceptron equivariant to the symmetry transformations of Entity-Relationship model. To this end, we identify the most expressive family of linear maps that are exactly equivariant to entity relationship symmetries, and further show that they subsume recently introduced equivariant maps for sets, exchangeable tensors, and graphs. The proposed feed-forward layer has linear complexity in the data and can be used for both inductive and transductive reasoning about relational databases, including database embedding, and the prediction of missing records. This provides a principled theoretical foundation for the application of deep learning to one of the most abundant forms of data. Empirically, EERN outperforms different variants of coupled matrix tensor factorization in both synthetic and real-data experiments.
\end{abstract}

\section{Introduction}
In the relational model of data, we have a set of \textit{entities}, and one or more \textit{instances} of each entity. 
These instances interact with each other through a set of fixed \textit{relations} between entities. 
A set of \textit{attributes} may be associated with each type of entity and relation.\footnote{An alternative terminology refers to instances and entities as entities and entity types.} 
This simple idea is widely used to represent data, often in the form of a relational database, across a variety of domains, from shopping records, social networking data, and health records, to heterogeneous data from astronomical surveys.  

In this paper we introduce provably maximal family of equivariant linear maps for this type of data. Such linear maps are then combined with a nonlinearity and stacked in order to build a deep Equivariant Entity Relationship Network (EERN), which is simply a constrained Multilayer Perceptron. To this end we first represent the relational data using a set of coupled sparse tensors. This is an alternative representation to the tabular form used in the relational database literature. In this representation, one could simultaneously permute instances of the same entity across all tensors using the same permutation, without affecting the content; \cref{fig:er} demonstrates this in our running example. These symmetries also identify the desirable equivariance conditions for the model-- \ie (only) such permutations of the input should result in the same permutation of the output of the layer. 

After deriving the closed form of equivariant linear maps, we use equivariant network for missing record prediction in both synthetic and real-world databases. Our baseline is the family of coupled tensor factorization methods that have a 
strict assumption on the data-generation process, since they assume each tensor is a product of factors. Our results suggest that even 
when this strict assumption holds, the equivariance assumption in a our parameterized deep model which has a learning component, could more effective, resulting in a better performance. 

\begin{figure*}[t]\centering
\includegraphics[width=\textwidth]{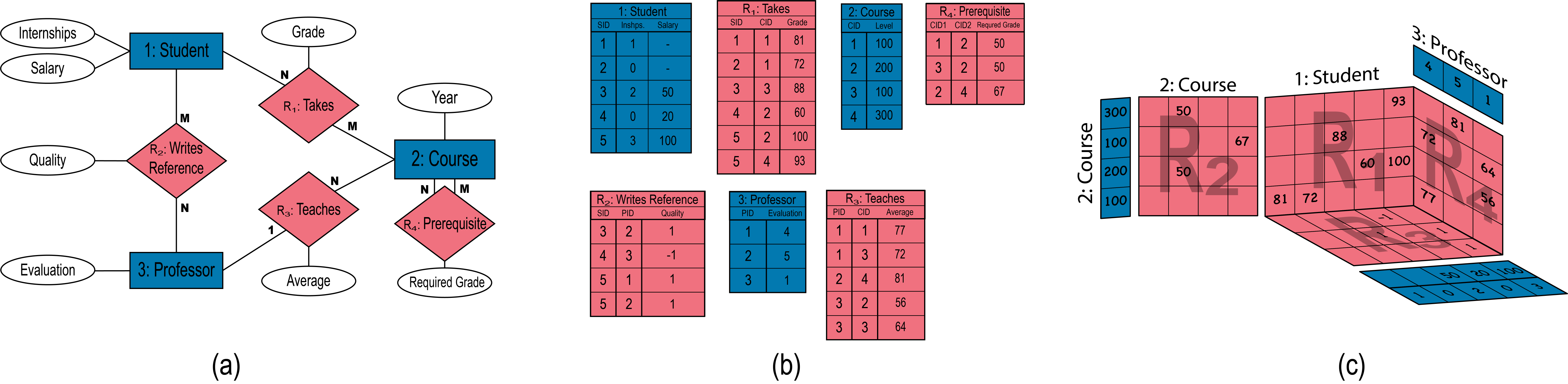}
\vspace{-2em}
\caption{{\footnotesize
(\textbf{a}) The Entity-Relationship (ER) diagram for our running example. There are three entities: \kw{Student}, \kw{Course} and \kw{Professor} (labeled 1,2 and 3 respectively), and four pairwise relations: \kw{Takes} (\kw{Student}-\kw{Course}, represented by $\RR_1 = \{1,2\}$); the self-self relation \kw{Prerequisite} (\kw{Course}-\kw{course}, represented by $\RR_2 = \{2,2\}$); \kw{Writes Reference} (\kw{Student}-\kw{Professor}, represented by $\RR_3 = \{1,3\}$); and \kw{Teaches} (\kw{Professor}-\kw{Course}, represented by $\RR_4 = \{2,3\}$). The full set of relations is $\RRR = \{\{1\}, \{2\}, \{3\}, \{1,2\} ,\{1,3\}, \{2,3\}, \{2, 2\} \}$. Both entities and relations have associated attributes  --- \eg when a \kw{Student} takes a \kw{Course}, they receive a \kw{Grade}. The singleton relation $\RR = {\{1\}}$ can encode \kw{Student} attribute(s) such as number of \kw{Internships}, or \kw{Salary} after graduating. Since each \kw{Course} has a single \kw{professor} as its teacher, this relation is one-to-many. (\textbf{b}) Some possible relational tables in an instantiation of the ER diagram of (a). There are $N_1=5$ instances of \kw{Student}, $N_2=4$ instances of \kw{Course} and $N_3=3$ instances of \kw{Professor}. The attributes associated with each entity and relation are stored in the corresponding table ---\eg table $\X^{\{1,2\}}$ suggests that the  \kw{Student} 5 took \kw{Course} 4 and received a \kw{Grade} of 93.
(\textbf{c}) Sparse tensor representation $\X^{{\RR}_1}, \ldots, \X^{\RR_{7}}$ of the tables of (b). The vectorized form of this set of sparse tensors, $\vec{\XX} = \tuple{\vec{\X^{\{1\}}};\ldots; \vec{\X^{\{2,2\}}}}$ ---the column-vector concatenation of $\vec{\X^{\RR}}$'s---  is the input to our feed-forward layer $\sigma(\W \vec{\XX})$. Here, the parameter-tying in the weight matrix $\W$ (\cref{fig:parameter_example}) guarantees that any permutation of elements of $\vec{\XX}$ corresponding to a shuffling of entities in this tensor representation (and only these permutations), results in the same permutation of the output of the layer.}
}
\label{fig:er}\label{fig:eg_tensors}\label{fig:tables}
\vspace{-1em}
\end{figure*}

\section{Related Literature}\label{sec:related}
Learning and inference on relational data has been the topic of research in machine learning over the past decades.
The relational model is closely related to first order and predicate logic, where the existence of a relation between instances becomes a truth statement about the world.
The same formalism is used in AI through a probabilistic approach to logic, where the field of \textit{statistical relational learning} has 
fused the relational model with the framework of \textit{probabilistic graphical models}~\cite{koller2009probabilistic}. 
Examples of such models include 
plate models, probabilistic relational models, Markov logic networks, and relational dependency networks~\citep{getoor2007introduction,raedt2016statistical}. Most relevant to our work in the statistical relational learning community is the relational neural network of~\citet{kazemi2017relnn}; see 
\ifbool{APX}{\cref{app:kazemi}}{Appendix A.1.2} for details.

An alternative to inference with symbolic representations of relational data is to use embeddings. In particular, \textit{tensor factorization} methods are extensively used for 
knowledge-graph embedding~\citep{nickel2016review}. 
Tensor factorization has been in turn extended to coupled matrix tensor factorization (CMTF)~\citep{yilmaz2011generalised} to handle multiple data sources jointly, by sharing latent factors across inputs when there is coupling of entities; see \ifbool{APX}{\cref{app:coupled}}{Appendix A.1.1} for more details. It is evident from \cref{fig:er}(c) that CMTF is directly comparable to our approach (in transductive setting), as it operates on the same data-structure. 

A closely related area that has enjoyed accelerated growth in recent years is \textit{relational and geometric deep learning}, where the term ``relational'' is used to denote the inductive bias introduced by 
a graph structure~\cite{battaglia2018relational}; see \ifbool{APX}{\cref{app:geometric}}{Appendix A.2} for more.
Although relational and graph-based terms are often used interchangeably in the machine learning community, they could refer to different data structures: graph-based models such as graph databases~\citep{robinson2013graph} and knowledge graphs 
simply represent data as an attributed (hyper-)graph, while the relational model, common in databases, uses the entity-relation (ER) diagram~\citep{chen1976entity} to constrain the relations of each instance (corresponding to a node), based on its entity-type; see \ifbool{APX}{\cref{app:stat-rel}}{Appendix A.1} for a discussion. 

Equivariant deep learning is another related area; see \ifbool{APX}{\cref{app:geometric,app:exchangeable}}{Appendices A.2 and A.3}.
Our model generalizes several equivariant layers proposed for structured domains: 
 it reduces to Deep-Sets \citep{zaheer_deepsets}
 when we have a single relation with a single entity; for example when we have only one blue table in \cref{fig:er}(b). 
\citet{hartford2018deep} consider a more general setting of interaction across different sets, such as user-tag-movie relations.
Our model reduce to theirs when we have a single relation involving multiple distinct entities -- \eg any of the pink tables / matrices in \cref{fig:er}, with the exception of $\RR_4$, which models the \kw{course-course} prerequisite interaction.
\citet{maron2018invariant} further allow repeated appearance of the same entity -- \eg in the \kw{node-node} relation of a graph, \kw{course-course} prerequisite relation.
 Our model reduces to this scheme when restricted to a single relation. For detailed discussion of these special cases see
 \ifbool{APX}{\cref{app:related}}{Appendix A}.

\section{Representations of Relational Data} \label{sec:representation}
We represent the relational model by a set of \textit{entities} $\mathcal{D} = \{1, ..., D \}$, and a set of \textit{relations} $\RRR \subseteq 2^{\mathcal{D}}$, indicating how the entities interact with one another.
For example, in \cref{fig:er}(a) we have $D=3$ entities, \kw{student}(1), \kw{course}(2), and \kw{professor}(3), where $\RRR = \{\{1\}, \{2\}, \{3\}, \{1,2\} ,\{1,3\}, \{2,3\}, \{2, 2\} \}$ is the set of relations.

For each entity $d \in \mathcal{D}$ we have a set $\{ 1, .., N_d \}$ of \textit{instances} of that entity; \eg $N_3 = 3$, is the number of \kw{professor} instances. For each relation $\RR \in \RRR$, where  $\RR = \{ d_1,\ldots, d_{|\RR|} \} $ is a set of entities, we observe data in the form of a set of tuples $\XX_\RR = \{\tuple{n_{d_1}, \ldots, n_{d_{|\RR|}}, x } \st n_{d_i}
\in [N_{d_i}], x \in \Real^K \}$. That is, each element of $\XX_\RR$ associates a feature vector $x$ with the relationship between the instances indexed by $n_{d_i}$ for each entity $d_i \in \RR$. 
 Note that the \textit{singleton relation} $\RR = \{d_i\}$ can be used for individual entity attributes (such as \kw{professors}' evaluations in \cref{fig:er}(a)). 
 For example, in \cref{fig:er}, we may have a tuple $\tuple{4,2,60} \in \XX_{\{1,2\}}$, which is identifying
 the grade of $60$, for \kw{student} $n_{1} = 4$ taking the \kw{course} $n_2 = 2$. The set of tuples $\XX_{\{1,2\}}$ therefore maintains the data associated with the \kw{student-course} relation $\RR_1 = \{1,2\}$.

In the most general case, we allow for both $\RRR$, and any $\RR \in \RRR$ to be multisets (\ie to contain duplicate entries). $\RRR$ is a multiset if we have multiple relations between the same set of entities. For example, we may have a \kw{supervises} relation between \kw{students} and \kw{professors}, in addition to the \kw{Writes Reference} relation. A particular relation $\RR$ is a multiset if it contains multiple copies of the same entity. Such relations are ubiquitous in the real world, describing for example, the connection graph of a social network, the sale/purchase relationships between a group of companies, or, in our running example, the \kw{Course}-\kw{Course} relation capturing prerequisite information. 
For our derivations we make the simplifying assumption that each \textit{attribute} $x \in \Re$ is a scalar. Extension to $\x \in \Re^{K}$ using multiple channels is simple and discussed in 
\ifbool{APX}{\cref{sec:multi_channels}}{Appendix F}.
Another common feature of relational data, the one-to-many relation, is addressed in 
\ifbool{APX}{\cref{sec:simplifications}}{Appendix E}.

\subsection{Tuples, Tables and Tensors}
In relational databases the set of tuples $\XX^{\RR}$ is often represented using a table, with one row for each tuple; see \cref{fig:tables}(b).
An equivalent representation for $\XX^{\RR}$ is using a ``sparse'' rank $|\RR|$ tensor $\X^\RR \in \Re^{{N_{d_1} \times \ldots \times N_{d_{|\RR|}}}}$, where each dimension of this tensor corresponds to an entity $d \in \RR$, and 
the the length of that dimension is the number of instances $N_{d}$. In other words
$$\tuple{n_{d_1}, \ldots, n_{d_{|\RR|}}, x} \in {\XX}^\RR \quad \Leftrightarrow \quad \X^{\RR}_{n_{d_1}, \ldots, n_{d_{|\RR|}}} = x.$$
We work with this tensor representation of relational data.
We use $\XX = \{ \X^\RR \st \RR \in \RRR\}$ to denote the set of all sparse tensors that define the relational data(base); see \cref{fig:eg_tensors}(c).
For the following discussions around exchangeability and equivariance, we assume that for all $\RR$, $\X^\RR$ are fully observed, dense tensors. Subsequently, we will discard this assumption and predict the missing records.
Note that relations $\RR$ can be multisets. For simplicity, we handle this in the main text by considering equal elements as distinct through indexing (\eg $d^{(i)}=d^{(j)}$), while leaving a formal definition of multisets for the supplementary material; see 
\ifbool{APX}{\cref{sec:multisets}}{Appendix B}. For example, the \kw{course-course} \kw{prerequisite} relation $\RR_{2,2} = \{2^{(1)}, 2^{(2)}\}$ in \cref{fig:er} is a multi-set.

\section{Symmetries of Relational Data}\label{sec:exchageabilities}
Recall that in the representation $\XX$, each entity $d \in \{1,...,D\}$ has a set of \emph{instances} indexed by $n_d \in \{1,..,N_d\}$. The ordering of $\{1,...,N_d\}$ is arbitrary, and we can shuffle these instances, affecting only the representation, and not the ``content'' of the relational data.
However, in order to maintain consistency across data tables,  we also have to shuffle all the tensors $\X^\RR$, where $d \in \RR$, using the same permutation applied to the tensor dimension corresponding to $d$; for example, in \cref{fig:er}(c), by permuting the order of five \kw{students} in the \kw{course-student} matrix, we also have to permute them in \kw{student-professor} matrix, and \kw{student} matrix (in blue). At a high level, this simple indifference to shuffling defines the exchangeabilities or symmetries of relational data. A mathematical group formalizes this idea. 

A mathematical \textit{group} is a set equipped with a binary operation, such that the set and the operation satisfy closure, associativity, invertability and existence of a unique identity element.  
$\grn{S}{M}$ refers to the \textit{symmetric group}, the group of all permutations of $M$ objects. A natural \textit{representation} for a member of this group $\g\prm{M} \in \grn{S}{M}$, is a permutation matrix $\G \in \{0,1\}^{M \times M}$. Here, the binary group operation is the same as the product of permutation matrices. 
In this notation, $\grn{S}{N_d}$ is the group of all permutations of instances $1,\ldots N_d$ of entity $d \in \{1,..,D\}$. To consider permutations to multiple dimensions of a data tensor we can use the direct product of groups.
Given two groups $\gr{G}$ and $\gr{H}$, the direct product $\gr{G} \times \gr{H}$ is defined by
\begin{align}
(\g, \gr{h}) \in \gr{G} \times \gr{H}  \Leftrightarrow \g \in \gr{G}, \gr{h} \in \gr{H} \quad \text{and} \quad
(\g, \gr{h}) \circ (\g', \gr{h}') = (\g \circ \g', \gr{h} \circ \gr{h}').
\end{align}
That is, the underlying set is the Cartesian product of the underlying sets of $\gr{G}$ and $\gr{H}$, and the group operation is element-wise.

Observe that we can associate the group $\grn{S}{N_1} \times \ldots \times \grn{S}{N_D}$ with the $D$ entities in a relational model, where each entity $d$ has $N_d$ instances. Intuitively, applying permutations from this group to the corresponding relational data should not affect the underlying content, while applying permutations from outside this group should. To see this, consider the tensor representation of \cref{fig:eg_tensors}(c): permuting students, courses or professors shuffles rows or columns of ${\XX}$, but preserves its underlying content. However, arbitrary shuffling of the elements of this tensor could alter its content. 
Our goal is to define a feed-forward layer that is ``aware'' of this structure. For this, we first need to formalize the \textit{action} of $\grn{S}{N_1} \times \ldots \times \grn{S}{N_D}$ on the ``vectorized''  $\XX$.

\paragraph{Vectorization.}
For each tensor $\X^\RR \in \XX$, $N_\RR = \prod_{d \in \RR} N_d$ refers to the total \textit{number of elements} of tensor $\X^\RR$ (note that for now we are assuming that the tensors are dense). We will refer to $N = \sum_{\RR \in \RRR} N_\RR$ as the number of elements of $\XX$.
Then $\vec{\X^\RR} \in \Real^{N_\RR}$ refers to the \textit{vectorization} of $\X^\RR$, obtained by successively stacking its elements along its dimensions, where the order of dimensions is given by $d \in \RR$. We use $\unvec{\cdot}$ to refer to the inverse operation of $\vec{\cdot}$, so that $\unvec{\vec{\X}} = \X$. With a slight abuse of notation, we use $\vec{\XX} \in \Real^{N}$ to refer to $[\vec{\X^{\RR_1}}; \ldots; \vec{\X^{\RR_{|\RRR|}}}]$, the vectorized form of the entire relational data. 
The weight matrix $\W$ that we define later creates a feed-forward layer $\sigma(\W \vec{\XX})$ applied to this vector.

\paragraph{Group Action.}
The \textit{action} of $\g \in \grn{S}{N_1} \times \ldots \times \grn{S}{N_D}$ on $\vec{\XX} \in \Real^N$, permutes the elements of $\vec{\XX}$. Our objective is to define this group action by mapping $\grn{S}{N_1} \times \ldots \times \grn{S}{N_D}$ to a group of permutations of all $N = \sum_{\RR \in \RRR} \prod_{d \in \RR} N_d$ entries of the database -- \ie a homomorphism into $\grn{S}{N}$.
To this end we need to use two types of matrix product.

Let $\G \in \Real^{N_1 \times N_2}$ and $\H \in \Real^{N_3 \times N_4}$ be two matrices. The \textit{direct sum} $\G \oplus \H$ is an $(N_1 + N_3) \times (N_2 + N_4)$ block-diagonal matrix, and the \textit{Kronecker product} $\G \otimes \H$ is an $(N_1 N_3) \times (N_2 N_4)$ matrix:
\begin{align*}
 \G \oplus \H \defeq   
 \begin{pmatrix*}
 \G & \mat{0} \\
 \mat{0} & \H 
 \end{pmatrix*}, \quad \G \otimes \H \defeq 
\begin{pmatrix*} 
\G_{1,1} \H & \hdots & \G_{1,N_2} \H \\
\vdots & \ddots & \vdots \\
\G_{N_1,1} \H & \hdots & \G_{N_1,N_2} \H
\end{pmatrix*}.
 \end{align*}
When both $\G$ and $\H$ 
are permutation matrices, $\G \oplus \H$ and $\G \otimes \H$ will also be permutation matrices.   Both of these matrix operations can represent the direct product of permutation groups. That is, given two permutation matrices $\G\prm{1}\in \grn{S}{N_1}$, and $\G\prm{2} \in \grn{S}{N_2}$, we can use both $\G\prm{1} \otimes \G\prm{2}$ and $\G\prm{1} \oplus \G\prm{2}$ to represent members of $\grn{S}{N_1} \times \grn{S}{N_2}$. However, the resulting permutation matrices, can be interpreted as different \textit{actions}: while the $(N_1 + N_2) \times (N_1 + N_2)$ direct sum matrix $\G\prm{1} \oplus \G\prm{2}$ is a permutation of $N_1 \mathbf{+} N_2$ objects,
the $(N_1 N_2) \times (N_1 N_2)$ Kronecker product matrix $\G\prm{1} \otimes \G\prm{2}$ is a permutation of $N_1 N_2$ objects.

\begin{mdframed}[style=MyFrame2]
\begin{claim}\label{claim:perm}
Consider the vectorized relational data $\vec{\XX}$ of length $N$. 
The action of $\grn{S}{N_1} \times \ldots \times \grn{S}{N_D}$ on $\vec{\XX}$ is given by the following permutation group
\begin{align}\label{eqn:perm}
\grn{G}{\XX} \defeq \big\{ \bigoplus_{\RR \in \RRR} \bigotimes_{d \in \RR} \G^d  \st (\G^{1},\ldots,\G^{D}) \in 
\grn{S}{N_1} \times \ldots \times \grn{S}{N_D} \big\}.
\end{align}
where the order of relations in $\bigoplus_{\RR \in \RRR}$ is consistent with the ordering used for vectorization of \XX.
\end{claim}
\vspace{-1em}
\end{mdframed}
\begin{proof}
The Kronecker product $\bigotimes_{d \in \RR} \G^d$ when applied to $\vec{\X^\RR}$, permutes the underlying tensor $\X^\RR$ along the axes $d \in \RR$.
Using direct sum, these permutations are applied to each tensor $\vec{\X^\RR}$ in $\vec{\XX} = [\vec{\X^{\RR_1}}; \ldots, \vec{\X^{\RR_D}}]$. The only constraint, enforced by \cref{eqn:perm} is to use the same permutation matrix $\G^d$ for all $\RR$ when $d \in \RR$.
Therefore any matrix-vector product $\G\prm{N} \vec{\XX}$ is a ``legal'' permutation of $\vec{\XX}$, since it only shuffles the instances of each entity.
\end{proof}

\section{Equivariant Linear Maps for Relational Data}\label{sec:layer}
Our objective is to constrain the standard feed-forward layer $f: \Real^{K \times N} \to \Real^{K' \times N}$ ---where $K,K'$ are the number of input and output channels, and $N = |\vec{\XX}|$--- such that
any ``legal'' transformation of the input, as defined in \cref{eqn:perm}, should result in the same transformation of the output. This is a useful bias because seeing one datapoint is equivalent to seeing all its legal transformations in an equivariant model -- therefore, the effect is similar to that of data-augmentation with exponentially many permutations of dataset. Using the equivariant map not only provides this exponential reduction in computation compared to data-augmentation, but also leads to a feed-forward layer with  $\mathcal{O}(N)$ implementation, in contrast with $\mathcal{O}(N^2)$ cost in a fully-connected layer.
For clarity, we limit the following definition to the case where $K=K' = 1$; see
\ifbool{APX}{\cref{sec:multi_channels}}{Appendix F} 
for the case of multiple channels.

\begin{mdframed}[style=MyFrame2]
\begin{definition}[Equivariant Entity-Relationship Layer; EERL]\label{def:layer}
Let $\G\prm{\XX}$ be any $N \times N$ permutation of $\vec{\XX}$. A fully connected layer $\sigma( \W \vec{\XX})$ with $\W \in \Real^{N \times N}$ is called an \textit{Equivariant Entity-Relationship Layer} if
\begin{align}\label{eqn:erl}
\sigma( \W \G\prm{\XX} \vec{\XX}) = \G\prm{\XX} \sigma( \W \vec{\XX}) \quad \forall \XX \quad  \Leftrightarrow \quad \G\prm{\XX} \in \grn{G}{\XX} .
\end{align}
That is, an EERL is a layer that 
commutes with the permutation $\G\prm{\XX}$ if and only if $\G\prm{\XX}$ is a legal permutation, as defined by \cref{eqn:perm}.
\end{definition}
\vspace{-1em}
\end{mdframed}
In its most general form, we can also allow for particular kinds of \textit{bias} parameters in \cref{def:layer}; see 
\ifbool{APX}{\cref{sec:bias}}{Appendix D}
for parameter-tying in the bias term. Since group $\gr{G}\prm{\XX}$ is finite, equivariance condition above 
can be satisfied using parameter-sharing~\cite{shawe1993symmetries,ravanbakhsh_symmetry}.
Next identify the closed form of parameter-sharing in $\W$ so as to guarantee EERL. 

Going back to the collection of tensors $\XX$ under legal permutations $\grn{G}{\XX}$, note that a legal permutation never moves an element of one tensor to another tensor, therefore each tensor $\X^{\RR_i}$ is an \emph{invariant} subset of the data (\ie a collection of orbits). One could decompose the equivariant map into linear maps between different invariant subsets.
This means that the matrix $\W \in \Real^{N \times N}$ has independent blocks $\W^{i,j} \in \Real^{N_{\RR_i} \times N_{\RR_j}}$ corresponding to each pair of relations $\RR_i, \RR_j$:
\begin{align}\label{eqn:param_blocks}
\W = 
\begin{pmatrix*}[l]
    \W^{1,1} & \W^{1,2} & \dots  & \W^{1,{|\RRR|}} \\
    \vdots & \vdots & \ddots & \vdots \\
    {\W^{{|\RRR|}, 1}} & {\W^{{|\RRR|}, 2}} & \dots  & \W^{{|\RRR|,|\RRR|}}
\end{pmatrix*}.
\end{align}
Moreover, if two relations have no common entity the only equivariant map between them is a constant map.
Our objective moving forward is to identify the form of each block that expresses the effect of observations in one relation (corresponding to tensors $\X^{\RR_i}$) on another. This form is rather involved, and we need 
to introduce an indexing notation before expressing it in closed form.

\paragraph{Indexing Notation.} 
The parameter block $\W^{i, j}$ is an $N_{\RR_{i}} \times N_{\RR_{j}}$ matrix, where $N_{\RR_{i}} = \prod_{d \in \RR_{i}} N_d$. We want to index rows and columns of $\W^{i,j}$.
Given the relation $\set{R}_i = \{d_1,\ldots, d_{|\set{R}_{i}|}\}$, we use the tuple $\n^{i} \defeq \tuple{n^i_{d_1}, \ldots, n^i_{d_{|\set{R}_i|}}} \in [N_{d_1}] \times \ldots \times [N_{d_{|\set{R}_i|}}]$ 
to index an element in the set $[N_{\RR_i}]$. Since each element of $\n^i$ indexes instances of a particular entity, $\n^{i}$
can be used as an index both for data block $\vec{\X^{\RR_i}}$ and for the rows of parameter block $\W^{i, j}$. In particular, to denote an entry 
of $\W^{i, j}$, we use $\W^{i,j}_{\n^{i}, \n^{j}}$. 
Moreover, we use $n^{i}_{d}$ to denote the element of $\n^{i}$ corresponding to entity $d \in \RR_i$. Note that
this is not necessarily the $d^{th}$ element of the tuple $\n^{i}$. For example, if $\RR_i = \{1,4,5\}$ and $\n^{i} = \tuple{400,12,3}$, then $n^{i}_{4} = 12$ and $n^{i}_{5} = 3$. When $\RR_i$ is a multiset, we can use $n^i_{d^{(k)}}$ to refer the to the element of $\n^i$ corresponding to the $k$-th occurrence of entity $d$ (where the order corresponds to the ordering of elements in $\n^i$). 
\ifbool{APX}{\cref{table:notation}}{Table 3}
in the Appendix summarizes our notation.

\begin{figure*}[t]\centering
\includegraphics[width=1\textwidth]{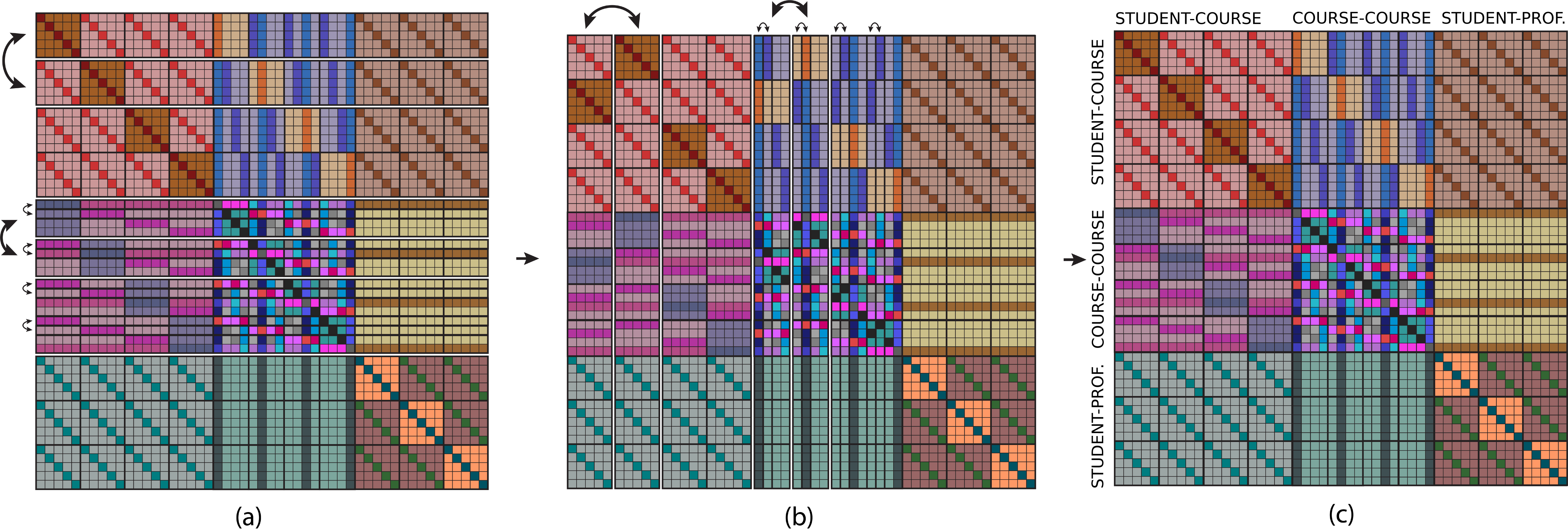}
\caption{{\footnotesize (\textbf{a}) . Each colour represents a unique parameter value. The nine blocks, showing the interaction of three relations \kw{student-course, course-course} and \kw{student-prof}, are clearly visible. The parameter-matrix commutes with valid permutations $\W \G\prm{\XX} = \G\prm{\XX} \W$, which means $\W = \G\prm{\XX} \W {\G\prm{\XX}}^\top$, that is $\W$ is invariant under simultaneous ``legal''
permutation of rows and columns.
The arrows indicate one such permutation that is being applied. (\textbf{b}) The result of applying a permutation to the rows of $\W$. 
The specific permutation only permutes 4 instances of \kw{course} by
swapping the first and second course. You can see that this permutation only
affects the first two blocks, where relations involve the \kw{course} entity.
(\textbf{c}) The result of applying the same permutation to the columns of $\W$. By doing so, as expected we recover the original matrix. This condition in $\W$ holds only for ``legal'' permutations as defined in \cref{eqn:perm}. 
This parameter matrix is based on Example 2 in the Appendix.
}}

\vspace{-1em}
\label{fig:parameter_example}
\end{figure*}

\subsection{Parameter Tying}\label{sec:parameter_tying}
Let  $\W^{i,j}_{\n^{i}, \n^{j}}$ and $\W^{i,j}_{{\m^{i}}, {\m^{j}}}$ denote two arbitrary elements of the 
parameter matrix $\W^{i,j}$. Our objective is to decide whether or not they should be tied together to ensure the resulting layer is an EERL (\cref{def:layer}). 
For this, we define an equivalence relation between the indices: $\tuple{\n^{i}, \n^{j}} \thicksim \tuple{\m^{i}, \m^{j}}$, and tie together all entries of $\W$ that are equivalent under this relation. Index $\tuple{\n^{i}, \n^{j}}$ is equivalent to index $\tuple{\m^{i}, \m^{j}}$ if they have the same ``equality patterns'' over their indices for each unique entity $d$.
We consider $\n^{i,j}$, the concatenation of $\n^{i}$ with $\n^{j}$, and examine the sub-tuples $\n_d^{i,j}$, where $\n^{i,j}$ is restricted to only those indices that index entity $d$. We do this so that we can ask if these indices refer to the same instance (\eg the same student), and we can only meaningfully compare indices of the same entity (\eg two entries in the \kw{student} table may refer to the same student, but an entry in the \kw{student} table cannot refer to the same thing as an entry in the \kw{course} table. Also, recall that we allow $\RR_i$ and $\RR_j$ to be multisets (\eg the \kw{course}-\kw{course} relation $\RR_4$ in \cref{fig:er}). 
We say two tuples $\n_d^{i,j} = \tuple{n_{d^{(1)}}, \ldots n_{d^{(L)}} }$ and $\m_d^{i,j} = \tuple{m_{d^{(1)}}, \ldots m_{d^{(L)}}}$ are equivalent iff the have the same equality pattern 
$n_{d^{(l)}} = n_{d^{(l')}} \Leftrightarrow m_{d^{(l)}} = m_{d^{(l')}} \forall l \in [L]$.
Accordingly, two index tuples
$\n^{i,j}$ and $\m^{i,j}$ are equivalent iff they are equivalent for all entities $d \in \RR_i \cup \RR_j$. 

Because of this tying scheme, the total number of \textit{free} parameters in $\W^{i,j}$ is the product of the number of possible different partitionings of $\n^{i,j}_d$ for each unique entity $d \in \RR_i \cup  \RR_j$, and so is a product of Bell numbers\footnote{The $k$-th Bell number counts the number of ways of partitioning a set of size $k$.},  
which count the number of partitions of a set given size; see 
\ifbool{APX}{\cref{sec:num_params}}{Appendix C}
for details.
This relation to Bell numbers was previously shown for equivariant graph networks~\citep{maron2018invariant} which, as discussed in \ifbool{APX}{\cref{app:maron}}{Appendix A.3.3}, are closely related to, and indeed a special case of, our model. 
This parameter-sharing scheme admits a simple recursive form, if the database has no self-relations (\ie the $\RR_i$ are not multisets); see 
\ifbool{APX}{\cref{sec:simplifications}}{Appendix E}.
In practice, we approach this constrained layer in a much more efficient way: the operations of the layer can be performed efficiently using pooling and broadcasting over tensors; see 
\ifbool{APX}{\cref{sec:simplifications}}{Appendix E}.

\begin{example}\label{eg:minimal_1_app}[\cref{fig:parameter_example}(a)]
To get an intuition for this tying scheme, consider a simplified version of the relational structure of \cref{fig:er}, restricted to the three relations $\RR_1 = \{1,2\}$, self-relation $\RR_2 = \{2,2\}$, and $R_3 = \{1,3\}$ with $N_1 = 5$ \kw{students}, $N_2 = 4$ \kw{courses}, and $N_3=3$ \kw{professors}. Then $N = 5 \times 4 + 4 \times 4 + 5 \times 3 = 51$, so $\W \in \Re^{51 \times 51}$ and will have nine blocks: $\W^{1,1} \in \Re^{20 \times 20}, \W^{1,2} \in \Re^{20 \times 16}. \W^{2,2} \in \Re^{16 \times 16}$ and so on.
We use tuple $\n^{1} = \tuple{n^1_1, n^1_2} \in [N_1] \times [N_2] = [5] \times [4]$ to index the rows and columns of $\W^{1,1}$. We also use $\n^{1}$ to index the rows of $\W^{1,2}$, and use $\n^{2} = \tuple{n^2_{2^{(1)}}, n^2_{2^{(2)}}} \in [N_2] \times [N_2] = [4] \times [4]$ to index its columns. Other blocks are indexed similarly. Suppose $\n^1 = \tuple{1, 4}, \n^2 = \tuple{4, 5}, \m^1 = \tuple{2, 3}$ and $\m^2 = \tuple{3, 2}$ and we are trying to determine whether $\W^{1,2}_{\n^1, \n^2}$ and $\W^{1,2}_{\m^1, \m^2}$ should be tied. Then $\n^{1,2} = \tuple{1,4,4,5}$ and $\m^{1,2} = \tuple{2,3,3,2}$. When we compare the sub-tuples restricted to unique entities, we see that the equality pattern of $\n^{1,2}_1 = \tuple{1}$ matches that of $\m^{1,2}_1 = \tuple{2}$ (since it is only a singleton, it matches trivially), and the equality pattern of $\n^{1,2}_2 = \tuple{4,4,5}$ matches that of $\m^{1,2}_2 = \tuple{3, 3, 2}$. So these weights should be tied. 
\end{example}

We now establish the optimality of our constrained linear layer for relational data.
\vspace{-.3em}
\begin{mdframed}[style=MyFrame2]
\begin{theorem}\label{thm:erl}
Let $\XX = \{ \X^\RR \st \RR \in \RRR\}$ be the tensor representation of some relational data and \vec{\XX} its vectorized form. $\sigma(\W\vec{\XX})$ is an Equivariant Entity-Relationship Layer (\cref{def:layer}) \emph{if and only if}
$\W$ is defined according to \cref{eqn:param_blocks}, with blocks that satisfy the tying scheme.
\end{theorem}
\end{mdframed}
\vspace{-1em}
Note that this theorem guarantees the equivariance of the feedforward layer, but it is also guarantees that the layer is not equivariant to any larger group $\gr{H}$, with $\grn{G}{\XX} < \gr{H} \leq \gr{S}\prm{N}$ (This is part of the definition of EERL); and further it guarantees \emph{maximality}, meaning that the proposed layer contains all the $\G\prm{\XX}$-equivariant feed-forward maps (this follows from  the only if direction of the theorem). \footnote{To see the distinction between properties 1 and 2, consider the example of 1D convolution: using a small kernel-size, the convolution layer satisfies 1, but not 2. There are permutation groups that only have layers satisfying 2 but not 1 -- \eg Alternating group with its standard action.}

\section{Experiments}\label{sec:experiments}
Our experiments compare EERN against two variants of coupled matrix tensor factorization (CMTF) for prediction of missing records on both synthetic and real-world data, and conduct ablation studies of EERN on synthetic data for inductive reasoning and side information utilization. Ablation studies appear in \ifbool{APX}{\cref{sec:ablation}}{Appendix I}.
We use a \textit{factorized auto-encoding} architecture consisting of a stack of EERL followed by pooling that produces emebeddings for each entity. The code is then fed to a decoding stack of EERL to reconstruct the sparse $\vec{\XX}$; see \ifbool{APX}{\cref{sec:experiments_details}}{Appendix H} for all the details.
We use two variants of CMTF~\citep{yilmaz2011generalised} as baselines for missing record prediction implemented through MATLAB Tensorlab package~\citep{vervliet2016tensorlab}: Coupled CP Factorization (C-CPF) and Coupled Tucker Factorization (C-TKF). C-CPF applies CANDECOMP/PARAFAC decomposition~\citep{hitchcock1927expression,carroll1970analysis,harshman1970foundations} to each input tensor/matrix while sharing latent factors across inputs when there is coupling of entities; C-TKF similarly applies Tucker decomposition~\citep{tucker1966some}.

\begin{table}
\vspace*{-3em}
\caption{\footnotesize{Test RMSE of synthetic data, for different sparsity levels, when the data-generation process is based on the assumption of Coupled CP Factorization (C-CPF) or Coupled Tucker Factorization (C-TKF).}}
\label{tab:mse}
\centering
\scalebox{.7}{
\begin{tabular}{r | c c | c c | c c}
\toprule
Sparsity $\rightarrow$ & \multicolumn{2}{c|}{10 \% } & \multicolumn{2}{c|}{50 \% } & \multicolumn{2}{c}{90 \%} \\ 
   Data Gen. $\rightarrow$ & CP  & Tucker & CP  & Tucker & CP & Tucker \\ 
Method $\downarrow$ & & & & & & \\
C-TKF & $1.720\pm0.025$          & $1.441\pm0.043$             & $1.096\pm0.005$          & $0.954\pm0.006$             & $0.531\pm0.031$          & $0.522\pm0.007$             \\ 
C-CPF     & $1.722\pm0.027$          & $1.411\pm0.044$             & $1.092\pm0.006$          & $0.968\pm0.009$             & $0.375\pm0.0004$          & $0.397\pm0.001$    
\\ 
EERN (ours)                        & $\mathbf{0.517}\pm0.046$ & $\mathbf{0.180}\pm0.013$    & $\mathbf{0.140}\pm0.012$ & $\mathbf{0.0824}\pm0.004$   & $\mathbf{0.101}\pm0.009$ & $\mathbf{0.0469}\pm0.001$   \\ 
\bottomrule
\end{tabular}
}
\end{table}
\subsection{Synthetic Data}
To continue with our running example we synthesize a toy dataset, restricted to $\RR_1 = \{1,2\}$ (\kw{student-course}),
$\RR_2 = \{1,3\}$ (\kw{student-professor}), and $\RR_3 = \{2,3\}$ (\kw{professor-course}), \footnote{The code for our experiments is available at 
<anonymous> (synthetic data) and 
<anonymous> (soccer data).} and evaluate on the RMSE of predicting missing entries of \kw{student-course} relation table. 
Each matrix $\X^{\{d_1,d_2\}} \in \Real^{N_{d_1} \times N_{d_2}}$ in the relational database, $\XX = \tuple{ \X^{\{1,2\}}, \X^{\{1,3\}}, \X^{\{2,3\}}}$, is produced by first uniformly sampling an h-dimensional embedding for each entity instance $\Z^{d} \in \Real^{N_d \times h}$, followed by either matrix product $\X^{\{d_1, d_2\}} := \Z^{d_1} {\Z^{d_2}}^{\mathsf{T}}$ or $\X^{\{d_1, d_2\}} := \Z^{d_1} {\mathbf{C}^{d_1d_2}} {\Z^{d_2}}^{\mathsf{T}}$, where $\mathbf{C}^{d_1d_2}$ is a random core matrix generated for each relation. These two operations reflect the assumed data generation process of Coupled CP and Tucker decomposition respectively, which we use as CMTF baselines.
A sparse subset of these matrices are observed and the missing values are predicted.
Table \ref{tab:mse} shows that EERN obtains significantly lower RMSE than both C-TKF and C-CPF on all sparsity levels even though the data generation process matches the assumption of C-TKF and C-CPF. 

\subsection{Real-World Data}

\paragraph{Soccer Data.} We use the European Soccer database\footnote{Dataset retrieved from \url{https://www.kaggle.com/hugomathien/soccer}} to build a simple relational model with three entities: \kw{player, team} and \kw{match}. The database contains information for about 11,000 \kw{players}, 300 \kw{teams} and 25,000 \kw{matches} in European soccer leagues. We extract a \kw{competes-in} relation between \kw{teams} and \kw{matches}, as well as a \kw{plays-in} relation between \kw{players} and \kw{matches} that identifies which players played in each match. 
Our objective is to predict whether the outcome of a match was \textit{Home Win}, \textit{Away Win}, or \textit{Draw}, as well as the score difference between home team and away team. 
A simple baseline is to always predict \textit{Home Win}, which obtains 46\% accuracy. By engineering features from temporal statistics (such as the result of recent games for a team relative to a particular target match, recent games two teams played against each other, as well as recent goal statistics) the best model reported on Kaggle achieve 56\% accuracy. 
Without using any temporal data, by simply taking the average for any such time series, our model achieves 53\% accuracy. This also matches the accuracy of professional bookies; see \cref{tab:real}. EERN also outperforms both C-TKF and C-CPF on predicting the goal difference.\footnote{Existing implementations for C-TKF and C-CPF only use square loss and therefore we could not use them for classification of the game outcome. Instead we defined the task for these methods to be prediction of goal difference between home and away, where we could use square loss.}

\begin{wraptable}{r}{.6 \textwidth}
\vspace*{-3em}
    \caption{\footnotesize{Performance on real-world relational data; predicting the outcome of games.}}
      \centering
      \scalebox{.8}{
        \begin{tabular}{r c c c c}
        \toprule
      & \multicolumn{2}{c}{Soccer (UEFA)} & \multicolumn{2}{c}{Hockey (NHL)} \\ \cline{2-5}
      & RMSE & Accuracy & RMSE & Accuracy \\ \midrule
C-TKF & 1.834 & - & 2.427 & - \\ 
C-CPF & 1.823 & - & 2.404& - \\ 
EERN (ours)  & \textbf{1.603} & \textbf{53\%} & \textbf{2.026} & \textbf{74\%} \\
\bottomrule
\end{tabular}
}
\label{tab:real}
\vspace*{-1em}
\end{wraptable}

\paragraph{Hockey Data.} We use an NHL hockey database\footnote{Dataset retrieved from \url{https://www.kaggle.com/martinellis/nhl-game-data}} to build a similar relational model to the previous experiment. The database contains information for about 2,212 \kw{players}, 33 \kw{teams} and 11,434 \kw{matches} in the NHL hockey league. We extract the same relations as for soccer data, our objective is to predict whether the outcome of a match was \textit{Home Win} or \textit{Away Win} (no draw in data), as well as the score difference between home and away team. The best model reported on Kaggle achieves 62\% accuracy, and our model achieves 74\% for predicting \textit{Home Win} or \textit{Away Win}.
Shown in Table \ref{tab:real}, EERN obtains significantly lower RMSE than CMTF baselines.

\section*{Conclusion and Future Work}\label{sec:discussion}
We have outlined a novel and principled approach to deep learning with relational data(bases).
In particular, we introduced a simple constraint in the form of tied parameters for the standard feed-forward layer and proved that any other tying scheme either ignores the exchangeabilities of relational data or can be obtained by further constraining our model. The proposed layer can be applied in inductive settings, where the relational databases used during training and test have no overlap. 
While our model enjoys a linear computational complexity in the size of the database, we have to overcome one more hurdle before applying this model to large-scale real-world databases: 
relational databases often hold large amount of data,
and in order for our model to be applied in these settings we need to perform mini-batch sampling. 
However, any such sampling has the effect of sparsifying the observed relations. A careful sampling procedure is required that minimizes this sparsification for a particular subset of entities or relations. 
While several recent works propose solutions to similar problems on graphs and tensors~\citep[\eg][]{hamilton2017inductive,hartford2018deep,ying2018graph,eksombatchai2017pixie,chen2018fastgcn,huang2018adaptive}, we leave this important direction for relational databases to future work.

\section*{Broader Impact}
Across various sectors, from healthcare to infrastructure, education to science and technology, often bigdata is stored in a relational database. Therefore the problem of learning and inference on this ubiquitous data structure is highly motivated.
To our knowledge, this paper is the first principled attempt at exposing this abundant form of data to our most successful machine learning paradigm, deep learning. While admittedly further steps are needed to handle large real-world databases,
the the theoretical framework discussed here could lead to development of software systems capable of inference and prediction on 
arbitrary relational databases, hopefully leading to a positive societal impact.



\bibliography{refs.bib}
\bibliographystyle{plainnat}

\ifbool{APX}{\appendix 

\section{A More Detailed Review of Related Literature}\label{app:related}
To our knowledge there are no similar frameworks for direct application of 
deep models to relational databases, and current practice is to 
automate feature-engineering for specific prediction tasks~\citep{lam2018learning}.

\subsection{Statistical Relational Learning and Knowledge-Graph Embedding}\label{app:stat-rel}
Statistical relational learning extends the reach of probabilistic inference to the relational model~\citep{raedt2016statistical}. 
In particular, a variety of work in \textit{lifted inference} procedures extends inference methods in graphical models to the relational setting, where in some cases the symmetry group of the model is used to speed up inference~\citep{kersting2012lifted}. 
Most relevant to our work from this community is the Relational Neural Network model of~\citet{kazemi2017relnn}; see \cref{app:kazemi}.

An alternative to inference with symbolic representations of relational data is to use embeddings. In particular, \textit{Tensor factorization} methods 
 that offer tractable inference in latent variable graphical models ~\citep{anandkumar2014tensor}, 
are extensively used for 
knowledge-graph embedding~\citep{nickel2016review}. 
A knowledge-graph can be expressed as an ER diagram with a single relation $\RR = \{1_{(1)}, 2, 1_{(2)}\}$, where $1_{(1)}, 1_{(2)}$  representing \kw{head} and \kw{tail} entities and $2$ is an entity representing the \kw{relation}. 
Alternatively, one could think of knowledge-graph as a graph representation for an instantiated ER diagram (as opposed to a set of of tables or tensors). However, in knowledge-graphs, an {entity-type} is a second class citizen, as it is either another attribute, or it is expressed through relations to special objects representing different ``types''. 
Therefore, compared to ER diagrams, knowledge-graphs are less structured, and more suitable for representing a variety of relations between different objects, where the distinction between entity types is not key. 
 
\subsubsection{Coupled Matrix Tensor Factorization} \label{app:coupled}
Tensor factorization has been extended to Coupled matrix tensor factorization (CMTF)~\citep{yilmaz2011generalised} to handle multiple data sources jointly, by sharing latent factors across inputs when there is coupling of entities. Structured data fusion~\citep{sorber2015structured} further extended traditional CMTF to handle certain transformation and regularization on factor matrices and support arbitrary coupling of input tensors. CMTF is previously used for topic modelling~\citep{bahargam2018constrained}, brain signal analysis~\citep{acar2017acmtf} and network analysis~\citep{ermics2015link} where joint analysis of data from different modes or sources enhances the signal, similar to the case of relational databases in this work.
 
 \subsubsection{Relation to RelNN of \citet{kazemi2017relnn}}\label{app:kazemi}
An alternative approach explored in the statistical relational learning community includes extensions of logistic regression to relational data~\citep{kazemi2014relational}, and further extensions to multiple layers~\citep{kazemi2017relnn}. The focus of these works is primarily on predicting properties of the various entity instances (the example they use is predicting a user's gender based on the ratings given to movies). 
 
 Their model works by essentially counting the number of instances satisfy a given properties, but is easiest understood by interpreting it as a series of convolution operations using row- and column-wise filters that capture these properties. 
 Consider Example 3 from~\citep{kazemi2017relnn} (also depicted in their Figure 4). We have a set of \kw{users} and a set of \kw{movies}, and there is a matrix $\mat{L} \in \Real^{N \times M}$, denoting which movies where liked by which users. As filters, they use binary vectors $\vect{a} \in \Real^M$, and $\vect{o} \in \Real^N$, representing which movies are \kw{action} and which users are \kw{old}, respectively. The task they pose is to predict the gender of a user\footnote{For simplicity, and to follow the example of~\citep{kazemi2017relnn}, we assume binary genders. However, we note that the real world is somewhat more complicated.}, given this information. To do so, they include a third filter, $\bm{\phi} \in \Real^{M}$, of learnable, ``numeric latent properties''. Each layer of their model then convolves each of these filters with the \kw{likes} matrix, then applies a simple linear scale and shift and a sigmoid activation. The result is three new filter vectors that can be used to make predictions or as filters in the next layer. For one layer, the outputs are then
 \begin{align*}
     \vect{v}^a = \sigma\bigg( w_0^a + w^a_1 \mat{L} \vect{a} \bigg), \;
     \vect{v}^o = \sigma\bigg( w_0^o + w^o_1 \vect{o}^T \mat{L} \bigg), \;
     \vect{v}^\phi = \sigma\bigg( w_0^\phi + w^\phi_1 \mat{L} \bm{\phi}  \bigg)
 \end{align*}
 where each $w$ is a scalar. Observe that, for example, the $n^{\text{th}}$ element of $\mat{L}\vect{a}$ counts the number of action movies liked by user $n$. Observe also that $\vect{v}^a, \vect{v}^\phi \in \Real^N$, while $\vect{v}^o \in \Real^M$. So $\vect{v}^a$ and $\vect{v}^\phi$ can be used to make predictions about individual users. Note that the number of parameters in their model grows both with the number of movies and with the number of layers in the network. 

Application of EERL to this example, reduces the 4 parameter model of \citep{hartford2018deep}. Indeed, most discussions and ``all'' experiments in \citet{kazemi2017relnn} assume a single relation. For completeness, we explain EERL in this setting.
Consider the \kw{Likes} matrix, and the \kw{Action} and \kw{Old} filters as tables. We predict the gender of the $n^{\text{th}}$ user, as
\begin{align}\label{eqn:ex_rlnn}
    \vect{g}_n = \sigma \bigg( w^{o}_0 \vect{o}_n + w^{o}_1 \sum_{n'=1}^N \vect{o}_{n'} + w^{L}_0 \sum_{m=1}^M \mat{L}_{n,m} + w^{L}_1 \sum_{n'=1}^M \sum_{m=1}^M \mat{L}_{n',m} + w^{a}_0 \sum_{m=1}^M \vect{a}_m \bigg).
\end{align}

The main difference between their model and ours is that they require per-item parameters (e.g., one parameter per movie), while, as can be seen from \cref{eqn:ex_rlnn}, the number of parameters in our model is independent of the number of instances and so does not grow with the number of users or movies (note that we have the option of adding such features to our model by having unique one-hot features for each \textsc{user} and \textsc{movie}.) As a result, our model can be applied in inductive settings as well. 
One may also draw a parallel between row and column convolution in \citet{kazemi2017relnn} with two out of four pooling operations when we have single relation between two entities. However these operations become insufficient when moving to models of self-relation (\eg 15 parameters for a single self-relation) and does cannot adequately capture the interaction between multiple relations as discussed in our provably optimal linear layer.





\subsection{Relational, Geometric and Equivariant Deep Learning}\label{app:geometric}
\citet{scarselli2009graph} introduced a generic framework that iteratively updates node embeddings using neural networks.  \citet{li2015gated} integrated this iterative process in a recurrent architecture.
\citet{gilmer2017neural} proposed a similar iterative procedure that updates node embeddings and messages between the neighbouring nodes, and show that it subsumes several other deep models for attributed graphs~\citep{duvenaud2015convolutional,schutt2017quantum,li2015gated,battaglia2016interaction,kearnes2016molecular}, 
including spectral methods. Their method is further generalized in \citep{kondor2018covariant} as well as \citep{maron2018invariant}, which is in turn subsumed in our framework.
Spectral methods extend convolution to graphs (and manifolds)
using eigenvectors of the Laplacian as the generalization of the Fourier basis~\citep{bronstein2017geometric,bruna2013spectral}. 
Simplified variations of this approach leads to an intuitive yet non-maximal 
parameter-sharing scheme that is widely used in practice~\citep{kipf2016semi}. 
This type of simplified graph convolution has also been used for relational reasoning with knowledge-graphs~\citep{schlichtkrull2018modeling}. See \cite{hamilton2017representation,battaglia2018relational} for a review of graph neural networks.

An alternative generalization of convolution
is defined for functions over groups, where, for finite groups this takes the form of  parameter-sharing~\citep{cohen2016group,ravanbakhsh_symmetry,shawe1989building}.
Moreover, convolution can be performed in the Fourier domain in this setting, where
irreducible representations of a group become the Fourier bases~\citep{kondor2018generalization}.
Particularly relevant to our work are the models of \citep{kondor2018covariant,maron2018invariant,albooyeh2019incidence} that operate on graph data using an equivariant design.
Equivariant deep networks for a variety of structured data are explored in several other recent works.
~\citep[\eg][]{worrall2017harmonic,cohen2018spherical,kondor2018clebsch,sabour2017dynamic,weiler20183d}; see also~\citep{cohen2016steerable,weiler2017learning,kondor2018covariant,anselmi2019symmetry}. 

\subsection{Parameter-Sharing, Exchangeability and Equivariance}\label{app:exchangeable}
The notion of invariance is also studied under 
the term exchangeability in statistics~\citep{orbanz2015bayesian}; see also \citep{bloemreddy}
for a probabilistic approach to equivariance. 
In graphical models exchangeability is often encoded through
plate notation, where parameter-sharing happens implicitly.
In the AI community, this relationship between the parameter sharing and ``invariance'' properties of the network was noticed in the early days of the Perceptron~\citep{minsky2017perceptrons, shawe1989building,shawe1993symmetries}.
This was rediscovered in \citep{ravanbakhsh_symmetry}, where this relation was leveraged  for equivariant model design. 

\subsubsection{Relation to Deep-Sets of \citet{zaheer_deepsets}}
propose an equivariant model for set data.
Our model reduces to their parameter-tying when we have a single relation $\RR = \{1\}$ with a single entity -- \ie $D = 1$; \ie a set of instances; see also 
\cref{example:recursive}
in 
\cref{sec:simplifications}.
Since we have a single relation, $\W$ matrix has a single block $\W = \W^{1,1}$, indexed by $n^{1}$. The elements of $\n^{1,1} = \n_{1}^{1,1} = \tuple{n_{1_{(1)}}, n_{1_{(2)}}}$ index the elements of this matrix, for entity $1$ (the only entity). There are two types of equality patterns $n_{1_{(1)}} = n_{1_{(2)}}$, and $n_{1_{(1)}} \neq n_{1_{(2)}}$, giving rise to the permutation equivariant layer introduced in~\citep{zaheer_deepsets}.

\subsubsection{Relationship to Exchangeable Tensor Models of \citet{hartford2018deep}}
\citet{hartford2018deep} consider a more general setting of interaction across different sets, such as user-tag-movie relations.
Our model produces their parameter-sharing when we have a 
 single relation $\RRR = \{\RR\}$ with multiple entities $D \geq 1$, where all entities appear only once -- \ie
$\kappa(d) = 1 \forall d \in \RR$. 
Here, again $\W = \W^{1,1}$, and $\n^{1,1}$, the concatenation of row-index $\n^{1}$ and column index $\n^{1}$, identifies an element of this matrix. 
Since each $d \in \RR$ has multiplicity $1$,
 $\n^{1,1}_{d} = \tuple{\n^{1,1}_{d_{(1)}}, \n^{1,1}_{d_{(2)}}} \;\forall d \in [D]$, and therefore $\n^{1,1}_{d}$ can have two class of equality patterns. This gives $2^D$ equivalence classes for $\n^{1,1}$, and therefore $2^D$ unique parameters for a rank $D$ exchangeable tensor.
 
 \subsubsection{Relationship to Equivariant Graph Networks of \citet{maron2018invariant}}\label{app:maron}
\citet{maron2018invariant} further relax the assumption of \cite{hartford2018deep}, and allow for $\kappa(d) \geq 1$. Intuitively, this form of relational data can model the interactions within and between sets; for example interaction within nodes of a graph is captured by an adjacency matrix, corresponding to $D = 1$ and $\RR = \{1,1\}$. This type of parameter-tying is maximal for graphs, and subsumes the parameter-tying approaches derived by simplification of Laplacian-based methods.
 When restricted to a single relation, our model reduces to the model of \citep{maron2018invariant}; however, when we have multiple relations, 
 $\W^{i,j}$ for $j \neq i$, our model captures the interaction between different relations / tensors.

\begin{table}[ht]
\begin{mdframed}[style=MyFrame]
\scalebox{.8}{
\begin{tabular}{c | c}
\begin{minipage}{0.6\linewidth}
\centering
\begin{tabular}{r l}
$\x, \n, \m$ & tuple or column vector \\
& (bold lower-case) \\
\tuple{\cdot, \cdot} & a tuple \\
$\X, \G, \W$ &  tensor, inc. matrix \\
& (bold upper-case) \\
$\XX, \RR$ & set (or multiset) \\
$\gr{S}, \gr{G}$ & group (caligraphic)\\
$\mathcal{D} = [D] = \{1,\ldots,D\}$ & set of entities \\
$N_1,\ldots, N_D$ & number of instances \\
$\RRR \subseteq 2^{\mathcal{D}}$ & a set of relations \\
$\RR_i \subseteq \mathcal{D}$ & a relation \\
$\X^{i} = \X^{\RR_i} \in \Re^{N_{d_1}\times \ldots, N_{d_{|\RR_i|}}}$ & data for a relation $\RR_i$\\
$\XX = \{\X^{i} \st \RR_i \in \RRR\}$ & relational data \\
$\vec{\X^{i}}$ & vectorization of $\X^i$ \\
\end{tabular}
\end{minipage}
& 
\begin{minipage}{0.6\linewidth}
\centering
\begin{tabular}{r l}
$\vec{\XX} $ & \tuple{\vec{\X^{1}}, \ldots, \vec{\X^{|\RRR|}}} \\
$N^{i} = \prod_{d \in \RR_i} N_d$ & length of $\vec{\X^{i}}$ \\
$N = \sum_{\RR_i \in \RRR} N^{i}$ & length of $\vec{\XX}$ \\
$\W \in \Re^{N \times N}$ & parameter matrix\\
$\W^{i,j} \in \Re^{N^i \times N^j}$ & (i,j)$^{th}$ block of $\W$\\
$\n^{i} = \tuple{n^i_{d_1}, \ldots, n^i_{d_{|\RR_i|}}}$ & index for $\vec{\X^i}$ \\ 
& and for rows of $\W^{i,j}$ \\
$\n = \tuple{\n^1,\ldots,\n^{|\RRR|}}$ & index for $\vec{\XX}$ \\ 
& and rows/columns of $\W$\\
$\grn{S}{M}$ & symmetric group $M$\\
\grn{G}{\XX} & group of $N \times N$ ``legal'' \\ 
& permutations  of $\vec{\XX}$ \\
& \\
\end{tabular}
\end{minipage}
\end{tabular}
}
\end{mdframed}
\caption{Summary of Notation}
\label{table:notation}
\end{table}

\section{Multiset Relations}\label{sec:multisets}
Because we allow a relation $\RR$ to contain the same entities multiple times, we formally define a multiset as a tuple $ \multiset{R} = \tuple{{\set{R}}, \kappa}$, where ${\set{R}}$ is a set, and $\kappa: \set{R} \to \Nat$ maps elements of ${\set{R}}$ to their multiset counts. We will call ${\set{R}}$ the \textit{elements} of the multiset $\multiset{R} = \tuple{{\set{R}}, \kappa}$, and $\kappa(d)$ the \textit{count} of element $d$. 
We define the union and intersection of two multisets $\RR_1$ and $\RR_2$ as
$\multiset{R}_1 \cup \multiset{R}_2 = \tuple{{\set{R}}_1 \cup {\set{R}}_2, \kappa_1 + \kappa_2}$  and
$\multiset{R}_1 \cap \multiset{R}_2 = \tuple{{\set{R}}_1 \cap {\set{R}}_2, \min\{\kappa_1, \kappa_2 \}}$.
In general, we may also refer to a multiset using typical set notation (\eg $\RR = \{ d_1, d_1, d_2 \}$). We will use bracketed superscripts to distinguish distinct but equal members of any multiset (\eg $\RR = \{ d_1, d_1, d_2 \} = \{ d_1^{(1)}, d_1^{(2)}, d_2^{(1)} \}$). The ordering of equal members is specified by context or arbitrarily. The size of a multiset $\multiset{R}$ accounts for multiplicities: $|\multiset{R}| = \sum_{d \in \set{R}} \kappa(d)$.

\section{Number of free parameters}\label{sec:num_params} For the multiset relations $ \multiset{R_i} = \tuple{{\set{R_i}}, \kappa}$ and $ \multiset{R_j} = \tuple{{\set{R_j}}, \kappa}$, recall that two parameters $\W^{i,j}_{\n^{i}, \n^{j}}$ and $\W^{i,j}_{{\m^{i}}, {\m^{j}}}$ are tied together if $\n^{i,j}$, the concatenation of $\n^{i}$ with $\n^{j}$, is in the same equivalence class as $\m^{i,j}$. We partition each $\n_d^{i,j}$ into sub-partitions $\set{P}(\n^{i,j}_d) \defeq \{ \set{P}_1, \ldots, \set{P}_L\}$ of indices whose values are equal, and consider $\n^{i,j}$ and $\m^{i,j}$ to be equivalent if their partitions are the same for all $d$:
\begin{align}
\n^{i,j} {\equiv} \m^{i,j}  \Leftrightarrow \set{P}(\n^{i,j}_d) = \set{P}(\m^{i,j}_d) \; 
\forall d \in \RR_i \cup \RR_j
\end{align}
See \ref{sec:parameter_tying} for details. This means that the total number of \textit{free} parameters in $\W^{i,j}$ is the product of the number of possible different partitionings for each unique entity $d \in \RR_i \cup  \RR_j$:
\begin{align}
|\w^{i,j}| = \prod_{d \in \RR_i \cup  \RR_j} b(\kappa^i(d) + \kappa^j(d))
\end{align}
where $\w^{i,j}$ is the free parameter vector associated with $\W^{i,j}$, and $b(\kappa)$ is the $\kappa^{th}$ Bell number, which counts the possible partitionings of a set of size $\kappa$.

\begin{example}\label{eg:minimal_1}[\cref{fig:parameter_example}(a)]
Consider again the simplified version of the relational structure of \cref{fig:er}, restricted to the three relations $\RR_1 = \{1,2\}$, self-relation $\RR_2 = \{2,2\}$, and $R_3 = \{1,3\}$ with $N_1 = 5$ \kw{students}, $N_2 = 4$ \kw{courses}, and $N_3=3$ \kw{professors}. 
We use tuple $\n^{1} = \tuple{n^1_1, n^1_2} \in [N_1] \times [N_2]$ to index the rows and columns of $\W^{1,1}$. We also use $\n^{1}$ to index the rows of $\W^{1,2}$, and use $\n^{2} = \tuple{n^2_{2^{(1)}}, n^2_{2^{(2)}}} \in [N_2] \times [N_2]$ to index its columns. Other blocks are indexed similarly. The elements of $\W^{1,1}$ take $b(2)b(2) = 4$ different values, depending on whether or not $n^1_1 = {n^1}'_1$ and $n^1_2 = {n^1}'_2$, for row index $\n^{1}$ and column index ${\n^{1}}'$ (where $b(k)$ is the k-th Bell number). The elements of $\W^{1,2}$ take $b(1)b(3) = 5$ different values: the index $n^1_1$ can only be partitioned in a single way ($b(1) = 1$). However index $n^1_2$ and indices $n^2_{2^{(1)}}$ and $n^2_{2^{(2)}}$ all index into the \kw{courses} table, and so can each potentially refer to the same course. We thus have a unique parameter for each possible combination of equalities between these three items, giving us a factor of $b(3) = 5$ different parameter values; see \cref{fig:parameter_example}(a), $\W^{1,1}$ is the upper left block, and $\W^{1,2}$ is the block to its right. 
The center block of \cref{fig:parameter_example}(a), $\W^{2,2}$ produces the effect of $\RR_2 = \{2,2\}$ on itself. Here, all four index values could refer to the same course, and so there are $b(4) = 15$ different parameters.
\end{example}

\section{Bias Parameters}\label{sec:bias}
For full generality, our definition of EERL could also include bias terms without affecting its exchangeability properties. We exclude these in the statements of our main theorems for the sake of simplicity, but discuss their inclusion here for completeness. For each relation ${\RR_i = \{d_1, ..., d_{|\RR_i|} \} }$, we define a bias tensor $\B^i \in \Real^{N_{d_1} \times ... \times N_{d_{|\RR_i|}}}$. The elements of $\B^i$ are tied together in a manner similar to the tying of elements in each $\W^{i,j}$: Two elements $\B^i_{\n^i}$ and $\B^i_{\m^i}$ are tied together iff ${\n^i \equiv \n^j}$, using the definition of equivalence from \cref{sec:parameter_tying}. Thus, we have a vector of additional free parameters $\b^i$ for each relation $\RR_i$, where
\begin{align}
    |\b^i| = \prod_{d \in {\RR_i}} b \big(\kappa^i(d) \big).
\end{align}

Consistent with our previous notation, we define $\BB = \{ \B^i \st \RR_i \in \RRR \}$, and $\vec{\BB} = \tuple{\vec{\B^1}, ..., \vec{\B^{|\RRR|}}}$. Then an EERL with bias terms is given by 
\begin{align}
    \vec{\Y} = \sigma \big( \W \vec{\XX} + \vec{\BB} \big) .
\end{align}
The following Claim asserts that we can add this bias term without affecting the desired properties of the EERL.
\begin{claim}\label{claim:bias}
If $\sigma \big( \W \vec{\XX}\big)$ is an EERL, then ${\sigma \big( \W \vec{\XX} + \vec{\BB} \big)}$ is an EERL. 
\end{claim}
The proof (found in \cref{proof:bias}) argues that, since $\sigma \big( \W \vec{\XX}\big)$ is an EERL, we just need to show that $\G^\XX \vec{\BB} = \vec{\BB}$ iff $\G^\XX \in \grn{G}{\XX}$, which holds due to the tying of patterns in each $\B^i$.

\section{Simplifications for Models without Self-Relations}\label{sec:simplifications}
In the special case that the multi relations $\multiset{R}_i$ and $\multiset{R}_j$ are sets ---\ie have no self-relations--- then the parameter tying scheme of \cref{sec:parameter_tying} can be simplified considerably. In this section we address some nice properties of this special setting.

\subsection{Efficient Implementation Using Subset-Pooling}\label{sec:pooling}
Due to the particular structure of $\W$ when all relations contain only unique entities, the operation $\W \vec{\XX}$ in the EERL can be implemented
using (sum/mean) pooling operations over the tensors $\X^\RR$ for $\RR \in \RRR$, without any need for vectorization, or for storing $\W$ directly.

For $\X^i \in \Re^{N_{d_1} \times \ldots \times N_{d_{|\RR|}}}$ and $\set{S} \subseteq \RR_i$, 
let $\pool(\X^i, \set{S})$ be the summation of the tensor $\X^i$
over the dimensions specified by $\set{S}$. That is, $\pool(\X^i, \set{S}) \in \Re^{N_{d'_1} \times \ldots \times N_{d'_{|\RR_i| - |\set{S}|}}}$ where $\RR_i \backslash \set{S} = \{d'_{1}, \ldots, d'_{|\RR| - |\set{S}|} \}$. Then we can write element $\n^i$ in the $i$-th block of $\W \vec{\XX}$ as
\begin{align}\label{eqn:pooling}
(\W \vec{\XX})_{(i,\n^i)} &= \sum_{\RR_j \in \RRR} \sum_{\set{S} \subseteq \RR_i \cap \RR_j} w^{i,j}_{\set{S}} \pool(\X^j, \set{S})_{\n^i_{\RR_i \backslash \set{S}}} 
\end{align}
where $\n^i_{\RR_i \backslash \set{S}}$ is the restriction of $\n^i$ to only elements indexing entities in $\RR_i \backslash \set{S}$. 
This formulation lends itself to a practical, efficient implementation where we simply compute each $\pool(\X^i, \set{S})$ term and broadcast-add them back into a tensor of appropriate dimensions.

\subsection{One-to-One and One-to-Many Relations}\label{sec:one_to_many}
In the special case of a one-to-one or one-to-many relations (\eg in \cref{fig:er}, one \kw{Professor} may teach many \kw{Courses}, but each \kw{Course} has only one \kw{Professor}), we may further reduce the number of parameters due to redundancies. Suppose $\RR_i \in \RRR$ is some relation, and entity $d \in \RR_i$ is in a one-to- relation with the remaining entities of $\RR_i$. Consider the 1D sub-array of $\XX^{\RR_i}$ obtained by varying the value of $\n^{i}_d$ while holding the remaining values $\n^{i}_{\RR_i \backslash \{d\}}$ fixed. This sub-array contains just a single non-zero entry. According to the tying scheme described in \cref{sec:parameter_tying}, the parameter block $\W^{i,i}$ will contain unique parameter values $w_{\RR_i}$ and $w_{\RR_i\backslash \{d\}}$. Intuitively however, these two parameters capture exactly the same information, since the sub-array obtained by fixing the values of $\RR_i \backslash \{d\}$ contains exactly the same data as the sub-array obtained by fixing the values of $\RR_i$ (\ie the same single value). More concretely, to use the notation of \cref{sec:pooling}, we have $\pool(\X^{i}, \RR_i \backslash \{d\})_{\n^i_{\{d\}}} = \pool(\X^{i}, \RR_i)$ in \cref{eqn:pooling}, and so we may tie $w^{i,i}_{\RR_i}$ and $w^{i,i}_{\RR_i \backslash \{d\}}$. 

 In fact, we can reduce the number of free parameters in the case of self-relations (\ie relations with non-unique entities) as well in a similar manner. 
 
\subsection{Recursive Definition of the Weight Matrix}
We are able to describe the form of the parameter matrix $\W^{i,j}$ concisely in a recursive fashion, using Kronecker products: For any $N \in \Nat$, let $\ones_N \in \Real^{N \times N}$ be the $N \times N$ matrix of all ones, and $\eye_N$ the $N \times N$ identity matrix. Given any set of (unique) entities $\set{S} = \{d_1, ..., d_{|\set{S}|} \} \subseteq \{1,..,D\}$, for $k = 1, ..., |\set{S}|$, recursively define the sets 
\begin{align}
    \WW^{\set{S}}_k = \bigg\{ \W \otimes \ones_{N_{d_k}} + \V \otimes \eye_{N_{d_k}} \st \W,\V \in \WW^{\set{S}}_{k-1} \bigg\} ,
\end{align}
with the base case of $\WW^{\set{S}}_0 = \Real$. Then for each block $\W^{i,j}$ of \cref{eqn:param_blocks} we simply have $\W^{i,j} \in \WW^{\RR_i \cap \RR_j}_{|\RR_i \cap \RR_j|}$. 

Writing the blocks of the matrix \cref{eqn:param_blocks} in this way makes it clear why block $\W^{i,j}$ contains $2^{|\RR_i \cap \RR_j|}$ unique parameter values in the case of distinct entities: at each level of the recursive definition we are doubling the total number of parameters by including terms from two elements from the level below. It also makes it clear that the parameter matrix for a rank-$k$ tensor is built from copies of parameter matrices for rank-$(k-1)$ tensors. 
\begin{example}\label{example:recursive}
 In the simple case where we have just one relation and one entity, $\RR = \{d\}$ and the parameter matrix is an element of ${\set{W}^\RR_1 = \{ w \otimes \ones_{N_d} + v \otimes \eye_{N_d} \st w,v \in \Real\}}$, which matches the parameter tying scheme of~\cite{zaheer_deepsets}. If instead we have a single relation with two distinct entities $\RR = \{ d_1, d_2 \}$, then the parameter matrix is an element of ${\set{W}^\RR_2 = \{ \W \otimes \ones_{N_{d_2}} + \V \otimes \eye_{N_{d_2}} \st \W,\V \in \set{W}^R_1\}}$, which matches the tying scheme of \cite{hartford2018deep}. 
\end{example}

\section{Using Multiple Channels}\label{sec:multi_channels}

Equivariance is maintained by composition of equivariant functions. This allows us to stack EERLs to build ``deep'' models that operate on relational databases. Using multiple input ($K$) and output ($K'$) channels is also possible by replacing the parameter matrix $\W \in \Re^{N \times N}$, with the parameter tensor $\W \in \Re^{K \times N \times N \times K'}$; while $K \times K'$ copies have the same parameter-tying pattern ---\ie there is no parameter-sharing ``across'' channels. The single-channel matrix-vector product in $\sigma(\W \vec{\XX})$ where $(\W \vec{\XX})_{\n}= \sum_{\n'} \W_{\n,\n'} \vec{\XX}_{\n'}$ is now replaced with contraction of two tensors $(\W \vec{\XX})_{\n,k'} = \sum_{\n',k \in [K]} \W_{k, \n,\n', k'} \vec{\XX}_{\n',k}$, for $k' \in [K']$. 

\section{Proofs}\label{app:proofs}

Observe that for any index tuple $\n$, we can express $\set{P}(\n_d)$ (\cref{sec:parameter_tying}) as

\begin{align}
    \set{P}(\n_d) = \Bigg\{ \bigg\{ d_{(k)}, d_{(l)} \; \forall k,l \st n_{d_{(k)}} = n_{d_{(l)}}  \bigg\}  \;\Big|\; d_{(k)} = d_{(l)} = d \Bigg\} .
\end{align}
We will make use of this formulation in the proofs below. 

\subsection{Proof of \cref{claim:bias}}\label{proof:bias}
\begin{proof}
We want to show that 
\begin{align}
    \G^\XX \sigma \big( \W \vec{\XX} + \vec{\BB} \big) = \sigma \big( \W \G^\XX \vec{\XX} + \vec{\BB} \big)
\end{align}
iff $\G^\XX \in \grn{G}{\XX}$. Since $\sigma \big( \W \vec{\XX}\big)$ is an EERL, this is equivalent to showing 
\begin{align}
    \G^\XX \vec{\BB} = \vec{\BB} \iff \G^\XX \in \grn{G}{\XX}. 
\end{align}
($\Longleftarrow$) Suppose $\G^\XX \in \grn{G}{\XX}$, with $\grn{G}{\XX}$ defined as in \cref{claim:perm}. Fix some relation $\multiset{\RR_i}$ and consider the $i$-th block of $ \G^\XX \vec{\BB}$:
\begin{align}
    \big( \G^\XX \vec{\BB} \big)_{(i, \n^i)} &= \big( \K^i \vec{\B^i} \big)_{\n^i} \\
    &= \sum_{{\n^i}'} \K^i_{\n^i, {\n^i}'} \vec{\B^i}_{{\n^i}'} \\
    &= \vec{\B^i}_{{\n^i}^*}, 
\end{align}
Where $\G^\XX = \diag( \K^1, ..., \K^{|\RRR|} )$, and ${\n^i}^*$ is the unique index such that $\K^i_{\n^i, {\n^i}^*} = 1$. As above, let $\n^i_d$ be the restriction of $\n^i$ to elements indexing entity $d$. Then we want to show that $\set{P}(\n^i_d) = \set{P}(\n^{i^*}_d)$ for all $d \in \RR_i$. Now, since $\G^\XX \in \grn{G}{\XX}$ and $\K^i_{\n^i, {\n^i}^*} = 1$ we have 
\begin{align}
    \G^d_{n^i_{d_{(k)}}, n^{i^*}_{d_{(k)}} } = 1
\end{align}
for all $d \in \RR_i$ and all $k$. That is $g^d \big(n^{i^*}_{d_{(k)}} \big) = n^i_{d_{(k)}}$. Consider $\set{S} \in \set{P}(\n^i_d)$. we have 
\begin{align}
    \set{S} &= \{ d_{(k)}, d_{(l)} \forall k,l \st n^i_{d_{(k)}} = n^i_{d_{(l)}} \} \\
    &= \{ d_{(k)}, d_{(l)} \forall k,l \st {g^d}^{-1}\big(n^i_{d_{(k)}}\big) = {g^d}^{-1}\big(n^i_{d_{(l)}} \big) \} \\
    &= \{ d_{(k)}, d_{(l)} \forall k,l \st n^{i^*}_{d_{(k)}} = n^{i^*}_{d_{(l)}} \}.
\end{align}
So $\set{P}(\n^i_d) \subseteq \set{P}(\n^{i^*}_d)$. A similar argument has $\set{P}(\n^i_d) \supseteq \set{P}(\n^{i^*}_d)$. Thus, we have $\vec{\B^i}_{{\n^i}^*} = \vec{\B^i}_{{\n^i}}$, which completes the first direction. 
\\
\\
$(\Longrightarrow)$ Let $\G^\XX \vec{\BB} = \vec{\BB}$. First, suppose for the sake of contradiction that $\G\prm{\XX} \not\in \bigoplus_{\RR \in \RRR} \grn{S}{N_{\RR}}$ and consider dividing the rows and columns of $\G\prm{\XX}$ into blocks that correspond to each relation $\RR_i$. Then since $\G\prm{\XX} \not\in \bigoplus_{\RR \in \RRR} \grn{S}{N_{\RR}}$, there exist $\RR_i, \RR_j \in \RRR$, with $i \neq j$ and $\n^i \in [N_{d^i_1}] \times ... \times [N_{d^i_{|\RR_i|}}]$ and $\n^j \in [N_{d^j_1}] \times ... \times [N_{d^j_{|\RR_j|}}]$ such that $\G^\XX$ maps $(i, \n^i)$ to $(j, \n^j)$. That is  $g^\XX\big( (i, \n^i) \big) = (j, \n^j)$ and thus $\G^\XX_{(j, \n^j), (i, \n^i)} = 1$. So 
\begin{align}
\big( \G\prm{\XX} \vec{\BB} \big)_{(j, \n^j)} &= \sum_{k}  \sum_{\n^k}  \G\prm{\XX}_{(j, \n^j), (k, \n^k)}  \vec{\BB}_{(k, \n^k)} \\
&= \vec{\BB}_{(i,\n^i)} \\
&= \vec{\B^i}_{\n^i} \\
&\not= \vec{\B^j}_{\n^j},
\end{align}
by the definition of $\BB$. And so $\G\prm{\XX} \vec{\BB} \neq \vec{\BB} $.
\\
\\
Next, suppose $\G\prm{\XX} \in \bigoplus_{\RR \in \RRR} \grn{S}{N_\RR}$. Then for all $\RR_i$, there exist $\K^i \in \grn{S}{N^i}$ such that $\G^\XX = \diag(\K^1, ..., \K^{|\RR_i|})$. For any $\n^i$ we have  
\begin{align}
    \vec{\B^i}_{\n^i} &= \big(\K^i \vec{\B^i}\big)_{\n^i} \\
    &= \sum_{{\n^i}'} \K^i_{\n^i, {\n^i}'} \vec{\B^i}_{{\n^i}'} \\
    &= \vec{\B^i}_{\n^{i^*}}
\end{align}
Where $\n^{i^*}$ is the unique index such that $\K^i_{\n^{i}, \n^{i^*}} = 1$. That is $k^i$ maps $\n^{i^*}$ to $\n^i$. Then by the definition of $\B^i$ we have 
\begin{align}\label{eqn:pipi}
    \set{P}(\n^{i^*}_d) &= \set{P}(\n^i_d) \nonumber\\
    &= \set{P}\big( (k^i(\n^{i^*}))_d \big)
\end{align}
for all $d \in \RR^i$. \cref{eqn:pipi} says that for each $d$ the action of $\K^i$ on elements of $\n^{i^*}$ is determined only by the values of those elements, not by the values of elements indexing other entities, and so $\K^i \in \bigotimes_{d \in \RR_i} \grn{S}{N_d}$. But \cref{eqn:pipi} also says that for all $k,l$
\begin{align}
    n^{i^*}_{d_{(k)}} = n^{i^*}_{d_{(l)}}  \iff \big(k^i(\n^{i^*})\big)_{d_{(k)}} = \big(k^i(\n^{i^*})\big)_{d_{(l)}} ,
\end{align}
which says that the action of $\K^i$ is the same across any duplications of $d$ (\ie $d_{(k)}$ and $d_{(l)}$), and so $\K^i = \bigotimes_{d \in \RR_i}  \G^d$, for some fixed $\G^d$, and therefore $\G^\XX \in \grn{G}{\XX}$.

\end{proof}

\subsection{\cref{lem:inter_matrix} and Proof}\label{proof:inter_matrix}


To prove our main results about the optimality of EERLs we require the following Lemma. 

\begin{lemma}\label{lem:inter_matrix}
For any permutation matrices $\K^i \in \grn{S}{N^i}$ and $\K^j \in \grn{S}{N^j}$ we have 
\begin{align*}
\K^{i} \W^{i, j} = \W^{i, j} \K^{j} \; \Leftrightarrow \; \K^{i} = \bigotimes_{d \in \multiset{R}_i} \G^d \text{ and } \K^{j} = \bigotimes_{d \in \multiset{R}_j} \G^d \quad \G^d \in \grn{S}{N_d}
\end{align*}
for constrained $\W^{i,j}$ as define above. 
That is $\K^i$ and $\K^j$ should separately permute the instances of each entity in the multisets $\multiset{R}_i$ and $\multiset{R}_j$, applying the same permutation to any duplicated entities, as well as to any entities common to both $\multiset{R}_i$ and $\multiset{R}_j$.
\end{lemma}


To get an intuition for this lemma, consider the special case of $i = j$. 
In this case, the claim is that $\W^{i, i}$ commutes with any permutation matrix that is of the form $\K^{i} = \bigotimes_{d \in \RR_i} \G^d$. This gives us the kind of commutativity we desire for an EERL, at least for the diagonal blocks of $\W$. Equivalently, commuting with $\K^{i}$ means that applying permutation $\K^{i}$ to the rows of $\W^{i, i}$ has the same effect as applying $\K^{i}$ to the columns of $\W^{i, i}$. In the case of $i \neq j$, ensuring that $\K^i$ and $\K^j$ are defined over the same underlying set of permutations, $\{\G^d \in \grn{S}{N_d} \st d \in \RR_i \cup \RR_j \}$, ensures that permuting the rows of $\W^{i, j}$ with $\K^i$ has the same effect as permuting the columns of $\W^{i, j}$ with $\K^j$. 
It is this property that will allow us to show that a network layer defined using such a parameter tying scheme is an EERL. See \cref{fig:parameter_example} for a minimal example, demonstrating this lemma.



We require the following technical Lemma for the proof of \cref{lem:inter_matrix}.
\begin{lemma}\label{lem:intersect}
Let $\RR_i, \RR_j \in \RRR$, and for each $d \in [D]$ let $\G^d \in \grn{S}{N_d}$. If $\G^{d}_{n^{i}_{d_{(k)}}, {n^{i}}'_{d_{(k)}} } = 1$ for all $d_{(k)} \in \RR_i$ with $d_{(k)} = d$, and $\G^{d}_{{n^{j}}'_{d_{(k)}}, n^{j}_{d_{(k)}} } = 1$ for all $d_{(k)} \in \RR_j$ with $d_{(k)} = d$, then for all $d_{(k)} \in \RR_i$ and $d_{(l)} \in \RR_j$ with $d_{(k)} = d_{(l)} = d$, we have ${n^{i}}'_{d_{(k)}} = n^{j}_{d_{(l)}}  \iff n^{i}_{d_{(k)}} = {n^{j}}'_{d_{(l)}}$.
\end{lemma}

\begin{proof}
Suppose $\G^{d}_{n^{i}_{d_{(k)}}, {n^{i}}'_{d_{(k)}} } = 1$ for all $d_{(k)} \in \RR_i$, and $\G^{d}_{{n^{j}}'_{d_{(k)}}, n^{j}_{d_{(k)}} } = 1$ for all $d_{(k)} \in \RR_j$. We prove the forward direction $(\Longrightarrow)$. The backward direction follows from an identical argument. Fix some $d_{(k)} \in \RR_i$ and $d_{(l)} \in \RR_j$ and suppose ${n^{i}}'_{d_{(k)}} = n^{j}_{d_{(l)}}$. By assumption we have $\G^{d}_{n^{i}_{d_{(k)}}, {n^{i}}'_{d_{(k)}} } = 1$ and so 
\begin{align}\label{eqn:map_kk}
    g^d({n^{i}}'_{d_{(k)}}) = n^{i}_{d_{(k)}} .
\end{align}
Similarly, we have $\G^{d}_{{n^{j}}'_{d_{(l)}}, n^{j}_{d_{(l)}} } = 1$ and so
\begin{align}\label{eqn:map_ll}
    g^d(n^{j}_{d_{(l)}}) = {n^{j}}'_{d_{(l)}} .
\end{align}
But ${n^{i}}'_{d_{(k)}} = n^{j}_{d_{(l)}}$ and substituting into \cref{eqn:map_kk} we have
\begin{align}\label{eqn:map_lk}
    g^d(n^{j}_{d_{(l)}}) = n^{i}_{d_{(k)}} .
\end{align}
And combining \cref{eqn:map_ll} and \cref{eqn:map_lk} gives $n^{i}_{d_{(k)}} = {n^{j}}'_{d_{(l)}}$, as desired. 

\end{proof}

We are now equipped to prove our main claims, starting with \cref{lem:inter_matrix}:

\begin{proof}
($\Longleftarrow$) Let $\multiset{\RR_i} = \{ d^i_1, ..., d^i_{|\multiset{\RR_i|}} \}$ and $\multiset{\RR_j} = \{ d^j_1, ..., d^j_{|\multiset{\RR_j|}} \}$ and fix some $\{\G^d \in \grn{S}{N_d} \st d \in \RR_i \cup \RR_j \}$. We index the rows of $\W^{i, j}$, and the rows and columns of $\K^{i}$, with tuples $\n^{i}, {\n^{i}}' \in [N_{d^i_1}] \times ... \times [N_{d^i_{|\RR_i|}}]$. Similarly, the columns of $ \W^{i, j}$, and rows and columns of $\K^{j}$, are indexed with tuples $\n^{j}, {\n^{j}}' \in [N_{d^j_1}] \times ... \times [N_{d^j_{|\RR_j|}}]$. Since $\K^i = \bigotimes_{d \in \RR_i} \G^d$ we have 
\begin{align*}
\K^{i}_{\n^{i}, {\n^{i}}'} = \prod_{d \in \multiset{\RR_i}} \G^{d}_{n^{i}_{d}, {n^{i}}'_{d} } = \prod_{d \in \RR_i} \prod_{k = 1}^{\kappa(d)} \G^{d}_{n^{i}_{d_{(k)}}, {n^{i}}'_{d_{(k)}} } .
\end{align*}
And thus,
\begin{align}\label{eqn:observe}
\K^{i}_{\n^{i}, {\n^{i}}'} = 1 \iff &\G^{d}_{n^{i}_{d_{(k)}}, {n^{i}}'_{d_{(k)}} } = 1, \forall d_{(k)} \in \multiset{\RR_i} \text{ s.t. } d_{(k)} = d . 
\end{align}
The same is true for $\RR_j$. That is
\begin{align}
\K^{j}_{{\n^{j}}', \n^{j}} = 1 \iff \G^{d}_{{n^{j}}'_{d_{(k)}}, n^{j}_{d_{(k)}} } = 1, \forall d_{(k)} \in \multiset{\RR_j} \text{ s.t. } d_{(k)} = d.
\end{align}
Now, fix some $\n^{i}$ and $\n^{j}$. Since $\K^{i}$ is a permutation matrix, and so has only one 1 per row, we have  
\begin{align}
\big( \K^{i} \W^{i, j} \big)_{\n^{i}, \n^{j}} &= \sum_{{\n^{i}}'}  \K^{i}_{\n^{i}, {\n^{i}}'} \, \W^{i, j}_{{\n^{i}}' \n^{j}} \nonumber \\
&= \W^{i, j}_{{\n^{i}}^*, \n^{j}}, \label{eqn:inter_mat1a}
\end{align}
where ${\n^{i}}^*$ is the (unique) element of $[N_{d^i_1}] \times ... \times [N_{d^i_{|R_i|}}]$ which satisfies $\G^{d}_{n^{i}_{d_{(k)}}, {n^{i}}^*_{d_{(k)}} } = 1$ for all $d_{(k)} \in \multiset{\RR_i}$ with $d_{(k)} = d$. Similarly, 
\begin{align}
\big( \W^{i, j} \K^{j} \big)_{\n^{i}, \n^{j}} &= \sum_{{\n^{j}}'} \W^{i, j}_{\n^{i}, {\n^{j}}'} \, \K^{j}_{{\n^{j}}', \n^{j}}  \nonumber \\
&=  \W^{i, j}_{\n^{i}, {\n^{j}}^*} \label{eqn:inter_mat1b}
\end{align}
where ${\n^{j}}^*$ is the (unique) element of $[N_{d^j_1}] \times ... \times [N_{d^j_{|\RR_j|}}]$ which satisfies $\G^{d}_{{n^{j}}^*_{d_{(k)}}, n^{j}_{d_{(k)}} } = 1$ for all $d_{(k)} \in \multiset{\RR_j}$ with $d_{(k)} = d$. 
\\
\\
We want to show that ${\set{P}(\n^{i^*,j}_d) = \set{P}(\n^{i,j^*}_d)}$ for all $d \in \RR_i \cup \RR_j$. Fix $d \in \RR_i \cup \RR_j$ and let $\multiset{S} \in \set{P}(\n^{i^*,j}_d)$. Then $\multiset{S} = \{ d_{(1)}, ..., d_{(K)} \}$, where $d_{(k)} = d$ for all $d_{(k)} \in \multiset{S}$, and for all $d_{(k)}, d_{(l)} \in \multiset{S}$, $n^{i^*,j}_{d_{(k)}} = n^{i^*,j}_{d_{(l)}}$. Then by \cref{lem:intersect} we have $n^{i,j^*}_{d_{(k)}} = n^{i,j^*}_{d_{(l)}}$, and so $\multiset{S} \in \set{P}(\n^{i,j^*}_d)$. So we have $\set{P}(\n^{i^*,j}_d) \subseteq \set{P}(\n^{i,j^*}_d)$, and the other containment follows identically by symmetry. So $\n^{i^*,j} \equiv \n^{i,j^*}$ by our definition in \cref{sec:parameter_tying}, and so  $\W^{i, j}_{{\n^{i}}^*, \n^{j}} = \W^{i, j}_{\n^{i}, {\n^{j}}^*}$ and by (\ref{eqn:inter_mat1a}) and (\ref{eqn:inter_mat1b}) above, $\K^{i} \W^{i, j} = \W^{i, j} \K^{j}$.

($\Longrightarrow$) Suppose $\K^i \W^{i,j} = \W^{i,j} \K^j$. Fix some $\n^i,\n^j$. Let $\n^{i^*}$ be the unique index such that $\K^i_{\n^i, \n^{i^*}} = 1$, and $\n^j$ the unique index such that $\K^j_{\n^{j^*}, \n^{j}} = 1$. Then 
\begin{align}
    \W^{i,j}_{\n^{i^*}, \n^j} &= \sum_{{\n^i}'} \K^i_{\n^i, {\n^i}'} \W^{i,j}_{{\n^i}', \n^j} \nonumber \\
    &= \big( \K^i \W^{i,j} \big)_{\n^i,\n^j} \nonumber \\
    &= \big( \W^{i,j} \K^j \big)_{\n^i,\n^j} \nonumber \\
    &= \sum_{{\n^j}'} \W^{i,j}_{{\n^i}, {\n^j}'} \K^j_{{\n^j}', {\n^j}} \nonumber \\
    &= \W^{i,j}_{\n^i, {\n^j}^*},
\end{align}
and so $\set{P}(\n^{i^*,j}_d) = \set{P}(\n^{i,j^*}_d)$ for all $d \in \RR_i \cup \RR_j$. But this implies that 
\begin{align}\label{eqn:nipi}
    \set{P}(\n^{i^*}_d) &= \set{P}(\n^{i}_d) \nonumber \\
    &= \set{P}\big( (k^i(\n^{i^*}))_d \big). 
\end{align}
\cref{eqn:nipi} says that for each $d$ the action of $\K^i$ on elements of $\n^{i^*}$ is determined only by the values of those elements, not by the values of elements indexing other entities, and so $\K^i \in \bigotimes_{d \in \RR_i} \grn{S}{N_d}$. But \cref{eqn:nipi} also means that for all $k,l$
\begin{align}
    n^{i^*}_{d_{(k)}} = n^{i^*}_{d_{(l)}}  \iff \big(k^i(\n^{i^*})\big)_{d_{(k)}} = \big(k^i(\n^{i^*})\big)_{d_{(l)}} ,
\end{align}
which says that the action of $\K^i$ is the same across any duplications of $d$ (\ie $d_{(k)}$ and $d_{(l)}$), and so $\K^i = \bigotimes_{d \in \RR_i}  \G^d$, for some fixed $\G^d$. Similarly, 
\begin{align}\label{eqn:njpj}
    \set{P}(\n^{j}_d) &= \set{P}(\n^{j^*}_d) \\
    &= \set{P}\big( (k^j(\n^{j}))_d \big), 
\end{align}
which shows that $\K^j = \bigotimes_{d \in \RR_j} {\G^{N_d}}'$. Finally, since $\set{P}(\n^{i^*,j}_d) = \set{P}(\n^{i,j^*}_d)$, we also have 
\begin{align}
    n^{i^*}_{d_{(k)}} = n^{j}_{d_{(l)}} &\iff n^{i}_{d_{(k)}} = n^{j^*}_{d_{(l)}}, 
\end{align}
for all $k,l$, which means 
\begin{align}
    n^{i^*}_{d_{(k)}} = n^{j}_{d_{(l)}} &\iff \big( k^i(\n^{i^*}) \big)_{d_{(k)}} = \big( k^j(\n^{j}) \big)_{d_{(l)}}\label{eqn:ij_equal}. 
\end{align}
\cref{eqn:ij_equal} says that $\K^i$ and $\K^j$ apply the same permutations to all duplications of any entities they have in common, and so $\G^d = {\G^d}'$, which completes the proof. 
\end{proof}

\subsection{Proof of \cref{thm:erl}}\label{proof:erl}
We break the proof into two parts, for the if ($\Rightarrow$) and only if 
($\Leftarrow$) statement.
\subsubsection{Proof of the if  statement ($\Rightarrow$) in \cref{thm:erl}}
\begin{proof}
Let $\G\prm{\XX} \in \gr{S}^N$ and $\gr{G}^\XX$ be defined as in \cref{claim:perm}. We need to show that $\G\prm{\XX} \sigma( \W \vec{\XX}) = \sigma( \W \G\prm{\XX} \vec{\XX})$ iff $\G\prm{\XX} \in \gr{G}^\XX$ for any assignment of values to the tables $\XX$. We prove each direction in turn. \\
\\
$(\Longrightarrow)$ We prove the contrapositive. Suppose $\G\prm{\XX} \not\in \gr{G}^\XX$. We first show that $\G\prm{\XX} \W \neq \W \G\prm{\XX}$ and then that $\G\prm{\XX} \sigma( \W \vec{\XX}) \neq \sigma( \W \G\prm{\XX} \vec{\XX})$ for an appropriate choice of $\XX$. There are three cases. First, suppose $\G\prm{\XX} \not\in \bigoplus_{\RR \in \RRR} \grn{S}{N_{\RR}}$ and consider dividing the rows and columns of $\G\prm{\XX}$ into blocks that correspond to the blocks of $\W$. Then since $\G\prm{\XX} \not\in \bigoplus_{\RR \in \RRR} \grn{S}{N_{\RR}}$, there exist $\RR_i, \RR_j \in \RRR$, with $i \neq j$ 
and $\n^i \in [N_{d^i_1}] \times ... \times [N_{d^i_{|\RR_i|}}]$ and $\n^j \in [N_{d^j_1}] \times ... \times [N_{d^j_{|\RR_j|}}]$ such that $\G^\XX$ maps $(i, \n^i)$ to $(j, \n^j)$. That is  $g^\XX\big( (i, \n^i) \big) = (j, \n^j)$ and thus, $\G^\XX_{(j, \n^j), (i, \n^i)} = 1$. And so 
\begin{align*}
\big( \G\prm{\XX} \W \big)_{(j, \n^j), (i, \n^i)} &= \sum_{k}  \sum_{\n^k}  \G\prm{\XX}_{(j, \n^j), (k, \n^k)}  \W_{(k, \n^k), (i, \n^i)} \\
&= \W^{i,i}_{\n^i,\n^i} ,
\end{align*}
by the definition of $\W$. Similarly, 
\begin{align*}
\big( \W \G\prm{\XX} \big)_{(j, \n^j), (i, \n^i)} &= \sum_{k}  \sum_{\n^k}  \W_{(j, \n^j), (k, \n^k)}  \G\prm{\XX}_{(k, \n^k), (i, \n^i)} \\
&= \W^{j,j}_{\n^j, \n^j} 
\end{align*}
But $\W^{i,i}_{\n^i,\n^i} \neq \W^{j,j}_{\n^j, \n^j}$ since $i \neq j$. And so $\G\prm{\XX} \W \neq \W \G\prm{\XX} $. 
\\
\\
Next, suppose $\G\prm{\XX} \in \bigoplus_{\RR \in \RRR} \grn{S}{N_\RR}$, but $\G\prm{\XX} \not\in \bigoplus_{\RR \in \RRR} \bigotimes_{d \in \RR} \grn{S}{N_d}$ and consider the diagonal blocks of $\G\prm{\XX} \W {\G\prm{\XX}}^T$ that correspond to those of $\W$. If $\G\prm{\XX} \in \bigoplus_{\RR \in \RRR} \grn{S}{N_R}$ then it is block diagonal with blocks corresponding to each $\RR \in \RRR$. But since $\G\prm{\XX} \not\in \bigoplus_{\RR \in \RRR} \bigotimes_{d \in \RR} \grn{S}{N_d}$, there exists some $\RR_i \in \RRR$ such that the $i^\text{th}$ diagonal block of $\G\prm{\XX}$ is not of the form $\bigotimes_{d \in \RR_i} \G^{d}$ for any $\G^{d}$. Then by \cref{lem:inter_matrix} we will have inequality between $\G\prm{\XX} \W {\G\prm{\XX}}^T$ and $\W$ in the $i^\text{th}$ diagonal block. 
\\
\\
Finally, suppose $\G\prm{\XX} \in \bigoplus_{\RR \in \RRR} \bigotimes_{d \in \RR} \grn{S}{N_d}$. Then $\G\prm{\XX} = \K^1 \oplus ... \oplus \K^{|\RRR|}$, where $\K^i \in  \bigotimes_{d \in \RR} \grn{S}{N_d}$ for all $i$. Since $\G\prm{\XX} \not\in \gr{G}^\XX$, there exist $\multiset{\RR_i} , \multiset{\RR_j} \in \RRR$, possibly with $i = j$, and a $d^* \in \RR_i \cap \RR_j$ such that
\begin{align*}
\K^i = \G^{d^i_1} \otimes ... \otimes \G^{d^*_{(k)}} \otimes ... \otimes \G^{d^i_{|\RR_i|}},
\end{align*}
and 
\begin{align*}
\K^j = \G^{d^j_1} \otimes ... \otimes \G^{d^*_{(l)}} \otimes ... \otimes \G^{d^j_{|\RR_j|}} ,
\end{align*}
but $ \G^{d^*_{(k)}} \neq \G^{d^*_{(l)}}$. Since $ \G^{d^*_{(k)}} \neq \G^{d^*_{(l)}}$ there exists $n \in [N_{d^*}]$ with $g^{d^*_{(l)}}(n) \neq g^{d^*_{(l)}}(n)$. Pick some $\n^i$ and $\n^j$ with $n^i_{d^*} = n^j_{d^*} = n$. Let ${\n^i}^*$ be the result of applying $\K^i$ to $\n^i$ and ${\n^j}^*$ the result of applying $\K^j$ to $\n^j$. Then we have 
\begin{align}\label{eqn:prod_1}
\big( \G\prm{\XX} \W \big)_{(i, {\n^i}^*), (j, \n^j)} &= \sum_{k}  \sum_{\n^k}  \G\prm{\XX}_{(i, {\n^i}^*), (k, \n^k)}  \W_{(k, \n^k), (j, \n^j)} \nonumber \\
&= \sum_{\n^k} \K^i_{{\n^i}^*, \n^k} \W^{i,j}_{\n^k, \n^j} \nonumber \\
&= \W^{i,j}_{\n^i, \n^j} ,
\end{align}
and
\begin{align}\label{eqn:prod_2}
\big( \W \G\prm{\XX} \big)_{(i, {\n^i}^*), (j, \n^j)} &= \sum_{k}  \sum_{\n^k}  \W_{(i, {\n^i}^*), (k, \n^k)}  \G\prm{\XX}_{(k, \n^k), (j, \n^j)}\nonumber \\
&= \sum_{\n^k}  \W^{i,j}_{{\n^i}^*,\n^k} \K^j_{\n^k, \n^j} \nonumber \\
&= \W^{i,j}_{{\n^i}^*, {{\n^j}^*}} .
\end{align}
Now, by construction we have $n^i_{d^*} = n^j_{d^*}$, but ${n^i}^*_{d^*} \not= {n^j}^*_{d^*}$. So $\set{P}(\n^{i,j}_{d^*}) \neq \set{P}(\n^{i^*,j^*}_{d^*})$ and therefore $\W^{i,j}_{\n^i, \n^j} \neq \W^{i,j}_{{\n^i}^*, {{\n^j}^*}}$. And so by \cref{eqn:prod_1} and \cref{eqn:prod_2} we have $\G\prm{\XX} \W \neq \W \G\prm{\XX}$.
\\
\\
And so in all three cases $\G\prm{\XX} \W \neq \W \G\prm{\XX}$. Thus, there exists some $\XX$, for which we have $\G\prm{\XX} \W \vec{\XX} \neq \W \G\prm{\XX} \vec{\XX}$. Since $\sigma$ is strictly monotonic, we have $\sigma(\G\prm{\XX} \W \vec{\XX}) \neq \sigma(\W \G\prm{\XX} \vec{\XX})$. And since $\sigma$ is element wise we have $\G\prm{\XX} \sigma(\W \vec{\XX}) \neq \sigma(\W \G\prm{\XX} \vec{\XX})$, which proves the first direction. 
\\
\\
$(\Longleftarrow)$ Suppose $\G\prm{\XX} \in \gr{G}^\XX$. That is, for all $d \in [D]$, let $\G^d \in \grn{S}{N_d}$ be some fixed permutation of $N_d$ objects and let $\G\prm{\XX} = \bigoplus_{\RR \in \RRR} \bigotimes_{d \in \RR} \G^d$. Observe that $\G\prm{\XX}$ is block-diagonal. Each block on the diagonal corresponds to an $\RR \in \RRR$ and is a Kronecker product over the matrices $\G^d$ for each $d \in \RR$. Let $\K^i = \bigotimes_{d \in \RR_i} \G^d$ for each $i \in [|\RRR|]$. That is,  
\begin{align*}
\G\prm{\XX} = 
\begin{bmatrix}
    {\K^{1}} & &  & \text{\huge0} \\
     & {\K^{2}}  &  & \\
     & & \ddots & \\
    \text{\huge0} & &  & {\K^{{|\RRR|}}}
\end{bmatrix} .
\end{align*}
And so the $i,j$-th block of $\G\prm{\XX} \W {\G\prm{\XX}}^T$ is given by:
\begin{align}\label{eqn:kwk}
\bigg( \G\prm{\XX} \W {\G\prm{\XX}}^T \bigg)^{i,j} &= \K^i \W^{i,j} {\K^j}^T \nonumber \\
&= \W^{i,j} .
\end{align}
The equality at \cref{eqn:kwk} follows from \cref{lem:inter_matrix}. Thus, we have $\G\prm{\XX} \W = \W {\G\prm{\XX}}$, and so for all $\XX$, $\G\prm{\XX} \W \vec{\XX} = \W {\G\prm{\XX}}\vec{\XX}$. Finally, since $\sigma$ is applied element-wise, we have 
\begin{align*}
\sigma( \W \G\prm{\XX} \vec{\XX}) &= \sigma( \G\prm{\XX} \W \vec{\XX}) \\
&= \G\prm{\XX} \sigma( \W \vec{\XX})
\end{align*}
Which proves the second direction. And so $\sigma(\W \vec{\XX})$ is an exchangeable relation layer, completing the proof. 
\end{proof}

\subsubsection{Proof of the only if direction ($\Leftarrow$) in \cref{thm:erl}}
The idea is that if $\W$ is not of the form \cref{eqn:param_blocks} then it has some block $\W^{i,j}$ containing two elements whose indices have the same equality pattern, but whose values are different. Based on these indices, we can explicitly construct a permutation which swaps the corresponding elements of these indices. This permutation is in $\gr{G}^\XX$ but it does not commute with $\W$. Now we present a detailed proof.  
\begin{proof}
Let $\G\prm{\XX} \in \gr{S}^N$. For any relation $\RR = \{d_1, ..., d_{|\RR|}  \} \in \RRR$, let $\set{N}^\RR = [N_{d_1}] \times ... \times [N_{d_{|\RR|}}]$ be the set of indices into $\XX^\RR$. If $\W$ is not of the form described in Section \ref{sec:layer} then there exist $i,j \in [|\RRR|]$, with $\n^i, {\n^i}' \in \set{N}^{\RR_i}$ and $\n^j, {\n^j}' \in \set{N}^{\RR_j}$ such that 
\begin{align}\label{eqn:param_unequal}
\W^{i,j}_{\n^i, \n^j} \neq \W^{i,j}_{ {\n^i}', {\n^j}'} 
\end{align}
but 
\begin{align}\label{eqn:inds_equal}
\set{P}(\n^{i,j}_d) = \set{P}(\n^{i',j'}_d), \; \forall d \in \RR_i \cup \RR_j
\end{align}
That is, the pairs $\n^i, \n^j$ and ${\n^i}', {\n^j}'$ have the same equality pattern over their elements, but the entries of $\W^{i,j}$ which correspond to these pairs have differing values, and thus violate the definition of $\W$ in Section \ref{sec:layer}. To show that the layer $\sigma(\W \vec{\XX})$ is not an EERL, we will demonstrate a permutation $\G\prm{\XX} \in \gr{G}^\XX$ for which $\G\prm{\XX}\W \neq \W \G\prm{\XX}$, and thus $\G\prm{\XX} \sigma( \W \vec{\XX}) \neq \sigma( \W \G\prm{\XX} \vec{\XX})$ for some $\XX$. 

Let $\G^{\XX} = \bigoplus_{\RR \in \RRR} \bigotimes_{d \in \RR} \G^d$, with the $\G^d$ defined as follows. For $d \in {\RR_i} \cap {\RR_j}$ and all $k$:
\begin{align*}
\G^d(n) = 
\begin{cases}
{n^i}'_{d_{(k)}} &\quad n = n^i_{d_{(k)}} \\
n^i_{d_{(k)}} &\quad n = {n^i}'_{d_{(k)}} \\
{n^j}'_{d_{(k)}} &\quad n = n^j_{d_{(k)}} \\
n^j_{d_{(k)}} &\quad n = {n^j}'_{d_{(k)}} \\
n &\quad \text{otherwise} 
\end{cases} .
\end{align*}
For $d \in {\RR_i} \backslash {\RR_j}$ and all $k$:
\begin{align*}
\G^d(n) = 
\begin{cases}
{n^i}'_{d_{(k)}} &\quad n = n^i_{d_{(k)}} \\
n^i_{d_{(k)}} &\quad n = {n^i}'_{d_{(k)}} \\
n &\quad \text{otherwise} 
\end{cases} .
\end{align*}
For $d \in {\RR_j} \backslash {\RR_i}$ and all $k$:
\begin{align*}
\G^d(n) = 
\begin{cases}
{n^j}'_{d_{(k)}} &\quad n = n^j_{d_{(k)}} \\
n^j_{d_{(k)}} &\quad n = {n^j}'_{d_{(k)}} \\
n &\quad \text{otherwise} 
\end{cases} .
\end{align*}
And for $d \not\in {\RR_i} \cup {\RR_j}$:
\begin{align*}
\G^d(n) = n .
\end{align*}
That is, each $\G^d$ swaps the elements of $\n^i$ with the corresponding elements of ${\n^i}'$, and the elements of $\n^j$ with those of ${\n^j}'$, so long as the relevant indices are present. For the case where $d \in {\RR_i} \cap {\RR_j}$, we need to make sure that this is a valid permutation. Specifically, we need to make sure that it is injective (it is clearly surjective from $[N_d]$ to $[N_d]$). But it is indeed injective, since we have $n^i_{d_{(k)}} = n^j_{d_{(k)}}$ iff ${n^i}'_{d_{(k)}} = {n^j}'_{d_{(k)}}$ for all $d \in {\RR_i} \cap {\RR_j}$ and all $k$, since $\set{P}(\n^{i,j}_d) = \set{P}(\n^{i',j'}_d)$.  

Now, for all $i$, let $\K^i = \bigotimes_{d \in \RR_i} \G^d$ be the $i^\text{th}$ diagonal block of $\G^\XX$. By definition of the $\G^d$, for all $d \in \RR_i$ we have $\G^d_{n^i_d, {n^i}'_d} = 1$, and thus by the observation at \cref{eqn:observe} we have $\K^i_{\n^i, {\n^i}'} = 1$. And so
\begin{align*}
\big( \G\prm{\XX} \W \big)_{(i, \n^i), (j, {\n^j}')} &= \sum_{k}  \sum_{\n^k}  \G\prm{\XX}_{(i, \n^i), (k, \n^k)}  \W_{(k, \n^k), (j, {\n^j}')} \\
&= \sum_{{\n^i}''} \K^i_{\n^i, {\n^i}''}  \W^{i,j}_{{\n^i}'', {\n^j}'} \\
&= \W^{i,j}_{{\n^i}', {\n^j}'} .
\end{align*}
Similarly, $\K^j_{\n^j, {\n^j}'} = 1$, so 
\begin{align*}
\big( \W \G\prm{\XX} \big)_{(i, \n^i), (j, {\n^j}')} &= \sum_{k}  \sum_{\n^k} \W_{(i, \n^i), (k, \n^k)} \G\prm{\XX}_{(k, \n^k), (j, {\n^j}')} \\
&= \sum_{{\n^j}''} \W^{i,j}_{\n^i, {\n^j}''} \K^j_{{\n^j}'', {\n^j}'} \\
&= \W^{i,j}_{\n^i, \n^j} .
\end{align*}
Finally, by \cref{eqn:param_unequal}, $\G\prm{\XX} \W \not= \W \G\prm{\XX}$, completing the proof. 

\end{proof}

\section{Details of Experiments}\label{sec:experiments_details}

\subsection{Synthetic Experiment}

\subsubsection{Data Generation}

The synthetic data we constructed used $200$ instances for each of the three entities, and the latent dimension of $h=2$ for ground-truth entity embeddings. We ensure that each row and column has at least 5 observations. We repeat 5 runs of 10\%, 50\% and 90\% train observation instances to reconstruct the rest of data as test observations, and report the average and standard deviation.

\subsubsection{Baselines}

For training, we create coupled tensors from the observed data and produce embeddings for each entity through C-CPF and C-TKF, which use the following objective function
\begin{equation}
    L = ||\mathbf{X}^{\{1,2\}} - \widehat{\mathbf{X}}^{\{1,2\}}|| + ||\mathbf{X}^{\{1,3\}} - \widehat{\mathbf{X}}^{\{1,3\}}|| + ||\mathbf{X}^{\{2,3\}} - \widehat{\mathbf{X}}^{\{2,3\}}||
\end{equation}
where for C-CPF the reconstruction is given by
\begin{equation}
    \widehat{\X}^{\{d_1, d_2\}} := \Z^{d_1} {\Z^{d_2}}^{\mathsf{T}} 
\end{equation}
and for C-TKF the generative process is
\begin{equation}
    \widehat{\X}^{\{d_1, d_2\}} := \Z^{d_1} {\mathbf{C}^{d_1d_2}} {\Z^{d_2}}^{\mathsf{T}}.
\end{equation}
Here, $||\cdot||$ is the Frobenius norm. We set latent factor dimension to 10 for CMTF baselines.
At test time, we use the decomposed entity embeddings $\Z$ (and core embeddings $\mathbf{C}$) to attempt to reconstruct test observations, reporting RMSE loss only on these test observations. 

\subsubsection{EERN}
For training, we pass our observed data through the network, producing encodings for each entity. We attempt to reconstruct the original input from these encodings. At test time, we use the training data to produce encodings as before, but now use these encodings to attempt to reconstruct test observations as well, reporting RMSE loss only on these test observations. 
Following~\citep{hartford2018deep}, we use a \textit{factorized auto-encoding} architecture consisting of a stack of EERL followed by pooling that produces code matrices $\Z^{d} \in \Real^{N_d \times h'} \quad \forall d \in \{1,2,3\}$ for each entity, \kw{student, course} and \kw{professor}. The code is then fed to a decoding EERN to reconstruct the sparse $\vec{\XX}$. The encoder consists of 7 EERLs, each with 64 hidden units, each using batch normalization~\citep{ioffe2015batch} and channel-wise dropout. We then apply mean pooling to produce encodings. The latent encoding dimension is set to 10 (same as baseline), dropout rate set to 0.2, activation function set to Leaky ReLU~\citep{xu2015empirical}, optimizer set to Adam~\citep{kingma2014adam}, number of epoches set to 4,000, and learning rate set to 0.0001. We found that batch normalization dramatically sped up the training procedure. 

\subsection{Real-world Experiment}

\subsubsection{Soccer Data Generation}

We use the European Soccer database to create relations $\RR_1 = \{1,2\}$ (\kw{match-team}), $\RR_2 = \{1,3\}$ (\kw{match-player}), with cardinality of the entities $N_{d_1} = 25,629; N_{d_2} = 288; N_{d_3} = 10,739$ and feature dimension $N_f = 135$. Their corresponding coupled feature tensors are $\X^{\{1,2\}}\in \RR^{N_{d_1} \times N_{d_2} \times N_f}$ and $\X^{\{1,3\}}\in \RR^{N_{d_1} \times N_{d_3} \times N_f}$. The prediction target vector for each match's \textit{Home} minus \textit{Away} score difference is a score vector $\mathbf{d} = \RR^{N_{d_1}}$. The prediction target for whether a match is \textit{Home Win}, \textit{Away Win} or \textit{Draw} is an indicator vector $\mathbf{r} = \{0,1,2\}^{N_{d_1}}$. We use 80\% of data for training, 10\% for validation and the rest 10\% for test.

\subsubsection{Hockey Data Generation}

We use the NHL Hockey database to create relations $\RR_1 = \{1,2\}$ (\kw{match-team}), $\RR_2 = \{1,3\}$ (\kw{match-player}), with cardinality of the entities $N_{d_1} = 11,434; N_{d_2} = 33; N_{d_3} = 2,212$ and feature dimension $N_f = 97$. Their corresponding coupled feature tensors are $\X^{\{1,2\}}\in \RR^{N_{d_1} \times N_{d_2} \times N_f}$ and $\X^{\{1,3\}}\in \RR^{N_{d_1} \times N_{d_3} \times N_f}$. The prediction target vector for each match's \textit{Home} minus \textit{Away} score difference is a score vector $\mathbf{d} = \RR^{N_{d_1}}$. The prediction target for whether a match is \textit{Home Win} or \textit{Away Win} is an indicator vector $\mathbf{r} = \{0,1\}^{N_{d_1}}$. We use 80\% of data for training, 10\% for validation and the rest 10\% for test.

\subsubsection{Baselines}

For training, we create coupled tensors from the observed data and produce embeddings for each entity through C-CPF and C-TKF, through the following objective function, where $||\cdot||$ denotes Frobenius norm, $[\mathbf{X}_1 \dots \mathbf{X}_n]$ denotes the Kruskal-form tensor created by factor matrices $\mathbf{X}_1 \dots \mathbf{X}_n$, and $\times_n$ denotes tensor mode-n product. We set latent factor dimension to 10 for CMTF baselines.
\begin{align}
    L = ||\mathbf{X}^{\{1,2\}} - \widehat{\mathbf{X}}^{\{1,2\}}|| + ||\mathbf{X}^{\{1,3\}} - \widehat{\mathbf{X}}^{\{1,3\}}|| + ||\mathbf{d} - \widehat{\mathbf{d}}||
\end{align}

where for C-CPF
\begin{align}
    \widehat{\mathbf{X}}^{\{1,2\}} := [\mathbf{Z}^1, \mathbf{Z}^2, \mathbf{Z}^{f_{12}}]\\
    \widehat{\mathbf{X}}^{\{1,3\}} := [\mathbf{Z}^1, \mathbf{Z}^3, \mathbf{Z}^{f_{13}}]  \\
    \widehat{\mathbf{d}} := [\mathbf{Z}^1, \mathbf{Z}^d] = \mathbf{Z}^1 {\mathbf{Z}^d}^{\mathsf{T}}
\end{align}

and for C-TKF
\begin{align}
    \widehat{\mathbf{X}}^{\{1,2\}} := \mathbf{C}^{12} \times_1 \mathbf{Z}^1 \times_2 \mathbf{Z}^2 \times_3 \mathbf{Z}^{f_{12}}\\
    \widehat{\mathbf{X}}^{\{1,3\}} := \mathbf{C}^{13} \times_1 \mathbf{Z}^1 \times_2 \mathbf{Z}^3 \times_3 \mathbf{Z}^{f_{13}}\\
    \widehat{\mathbf{d}} := \mathbf{C}^{d} \times_1 \mathbf{Z}^1 \times_2 \mathbf{Z}^d = \mathbf{Z}^1 \mathbf{C}^{d} {\mathbf{Z}^d}^{\mathsf{T}}
\end{align}

At test time we use the decomposed entity embeddings $\mathbf{Z}$ (and core embeddings $\mathbf{C}$) to attempt to reconstruct test observations, reporting RMSE loss only on these test observations.

\subsubsection{EERN}
For training, we pass our observed data through the network, producing encodings for each entity. Then we predict $\mathbf{d}$ or $\mathbf{r}$ from \kw{match} encodings respectively through minimizing RMSE loss and cross entropy loss. At test time, we use the training data to produce encodings as before, but now use these encodings to predict test target, reporting RMSE loss and accuracy only on these test observations. 
We use a \textit{factorized auto-encoding} architecture consisting of a stack of EERL followed by pooling that produces code matrices $\Z^{d} \in \Real^{N_d \times h'} \quad \forall d \in \{1,2,3\}$ for each entity, \kw{match, team} and \kw{player}. The code is then fed to a decoding EERN to predict $\mathbf{d}$ or $\mathbf{r}$ through their respective losses. We use the same architecture as the synthetic experiments but set hidden units to 40, activation function to ReLU, dropout rate to 0.5, epochs to 10,000 for soccer experiment and 2,000 for hockey experiment, intermediate pooling to mean, final pooling to sum, and learning rate to 0.001.

\section{Ablation Study with Synthetic Data}\label{sec:ablation}
We qualitatively evaluate EERNs using synthetically-generated data. In this way we can examine both the quality of the latent embeddings produced, and the ability of EERNs to use information from across the database to make predictions for a particular table. We do these in both the transductive and inductive settings. 

\subsection{Data Generation}
The data is constructed using $200$ instances for each of the three entities (\kw{Student, Course, Professor}). For each instance we sample $h = 2$ random values representing its ground-truth embedding. The entries of the data tables are then produced as the inner product between the corresponding row and column embeddings.

\subsection{Embedding Visualization}
\begin{figure}[t]\centering
\hbox{
\subfigure[{Transductive Truth}]{\includegraphics[width=0.24\columnwidth]{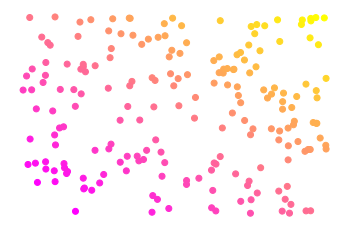}}
\subfigure[{Transductive Pred.}]{\includegraphics[width=0.24\columnwidth]{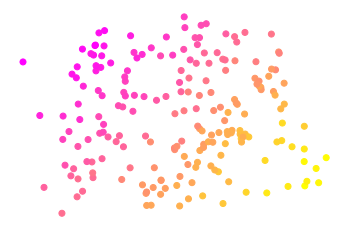}}
\subfigure[Inductive Truth]{\includegraphics[width=0.24\columnwidth]{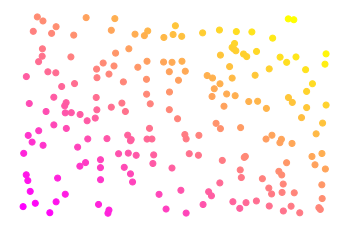}}
\subfigure[Inductive Pred.]{\includegraphics[width=0.24\columnwidth]{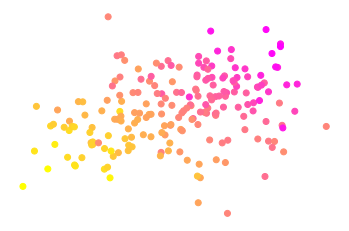}}
}
\caption{\footnotesize{ Ground truth versus predicted embedding for \kw{Student} instances in the transductive (a), (b) and inductive (c), (d) settings. The x-y location of each point encodes the prediction $\widehat{\Z^d}$, while the coloring for both transductive and inductive setting is consistent across ground truth and prediction. In the inductive setting, training and test databases contain completely distinct instances (i.e., completely different students).
}}\label{fig:embedding}
\end{figure}

To examine the quality of the learned embeddings we set the input and predicted embedding dimensions to be the same ($h = h'=2$), and train EERN to reconstruct unobserved entries for tables generated with one set of latent embeddings, and thus visualize predicted embedding generated by the trained model for both tables created by original input embedding (transductive) and unseen new input embedding (inductive). We can then visualize the relationship between the input and predicted embeddings in both transductive and inductive setting. \cref{fig:embedding}(a) and (b) show this relationship and suggest that the learned embeddings agree with the ground truth in the transductive setting: the same coloring was applied for input and predicted embedding and points with similar colors share vicinity with each other in both input and predicted embedding. We see a similar relationship in the inductive setting in \cref{fig:embedding}(c) and (d) where the model has not seen the \kw{student} instances before but the inductive input and predicted embedding using the same coloring show points with similar color in the same vicinity. This suggests that the model is still able to produce a reasonable embedding in the inductive setting. Note that in the best case, the inductive input versus predicted embeddings can agree up to a diffeomorphism.

\subsection{Missing Record Prediction}

\begin{figure}[t]
\centering
\subfigure[Transductive]{\includegraphics[width=.3\columnwidth]{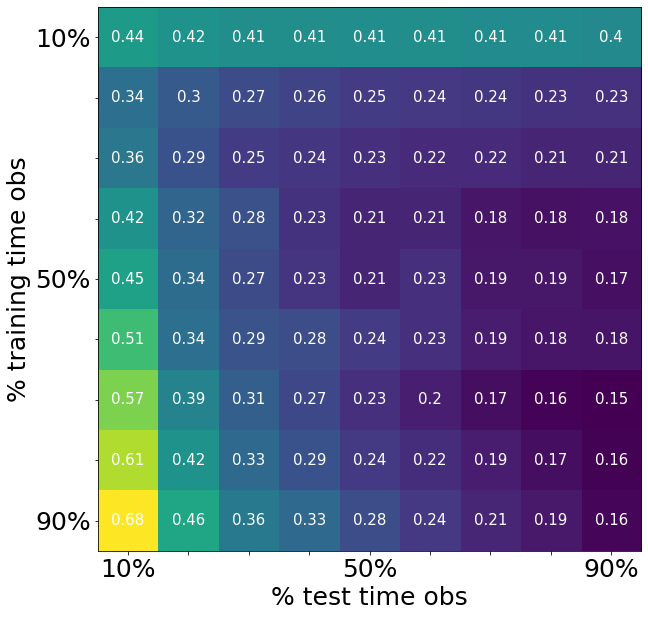}}
\subfigure[Inductive]{\includegraphics[width=0.3\columnwidth]{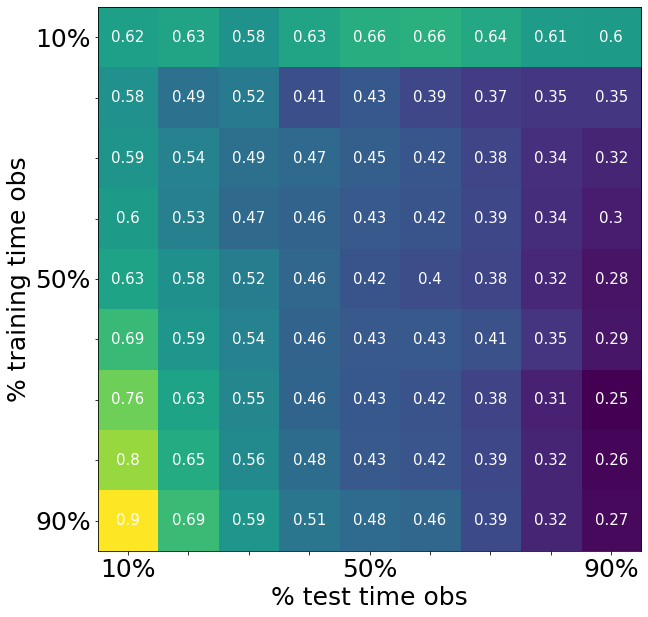}}
\caption{\footnotesize{
 Average root mean squared error in predicting missing records in \kw{student-course} as a function of sparsity level of the whole database $\XX$ during training (x-axis) and test (y-axis), in the transductive setting (c) and the inductive setting (d). In (d) the model is tested on a new database with \kw{students, courses} and \kw{professors} unseen during training time. The baseline is predicting the mean value of training observations. At test time, the observed entries are used to predict the values of the fixed, held-out test set.}}\label{fig:sparsity}
\end{figure}

Once we finish training the model, we can apply it to another instantiation --- that is a dataset with completely different \kw{students, courses} and \kw{professors}.
This is possible because the unique values in $\W$ do not grow with the number of instances of each entity in our database -- \ie we can resize $W$ by repeating its pattern to fit the dimensionality of our data. In practice, since we use pooling-broadcasting operations, this resizing is implicit. \cref{fig:sparsity}(b) shows the results for missing record prediction experiment in the inductive setting. 

Importantly, by incorporating enough new data at test time, the model can achieve the same level of performance in inductive setting as in the transductive setting. This can have interesting real-world applications, as it enables transfer learning across databases and allows for predictive analysis without training for new entities in a database as they become available.

Next, we set out to predict missing records in the \kw{student-course} table using observed data from the whole database. For this, the factorized auto-encoding architecture is trained to only minimize the reconstruction error for ``observed'' entries in the \kw{student-course} tensor. We set the latent encoding size to $h'=10$. We first set aside a special 10\% of the data that is only ever used for testing. Our objective will be to predict these missing entries, and we will vary 1) the proportion of data used to train the model, and 2) the proportion of data used to make predictions at test time. At test time, the data given to the model will include new data that was unobserved during training, allowing us to gauge the model's ability to incorporate new information without retraining. 

We train nine models using respectively 10\% to 90\% of the data as observed. The observed data is passed through the network, producing latent encodings for each entity instance. We then attempt to reconstruct the original input from these encodings. For consistency of comparison, the subsets of observed entries are chosen so as to be nested. At test time, we produce latent encodings as before, but now use them to reconstruct test observations, reporting RMSE only on these. 

\cref{fig:sparsity}(a) visualizes the prediction error of the model, averaged over 5 runs. The $y$-axis shows the proportion of data the model was trained on, and the $x$-axis shows the proportion provided at test time. Naturally, seeing more data during training helps improve predictions, but we particularly note that predictions can also be improved by incorporating new data at test time, without expensive re-training. 

\begin{wrapfigure}[10]{r}{.4\textwidth}\centering
\includegraphics[width=0.4\textwidth]{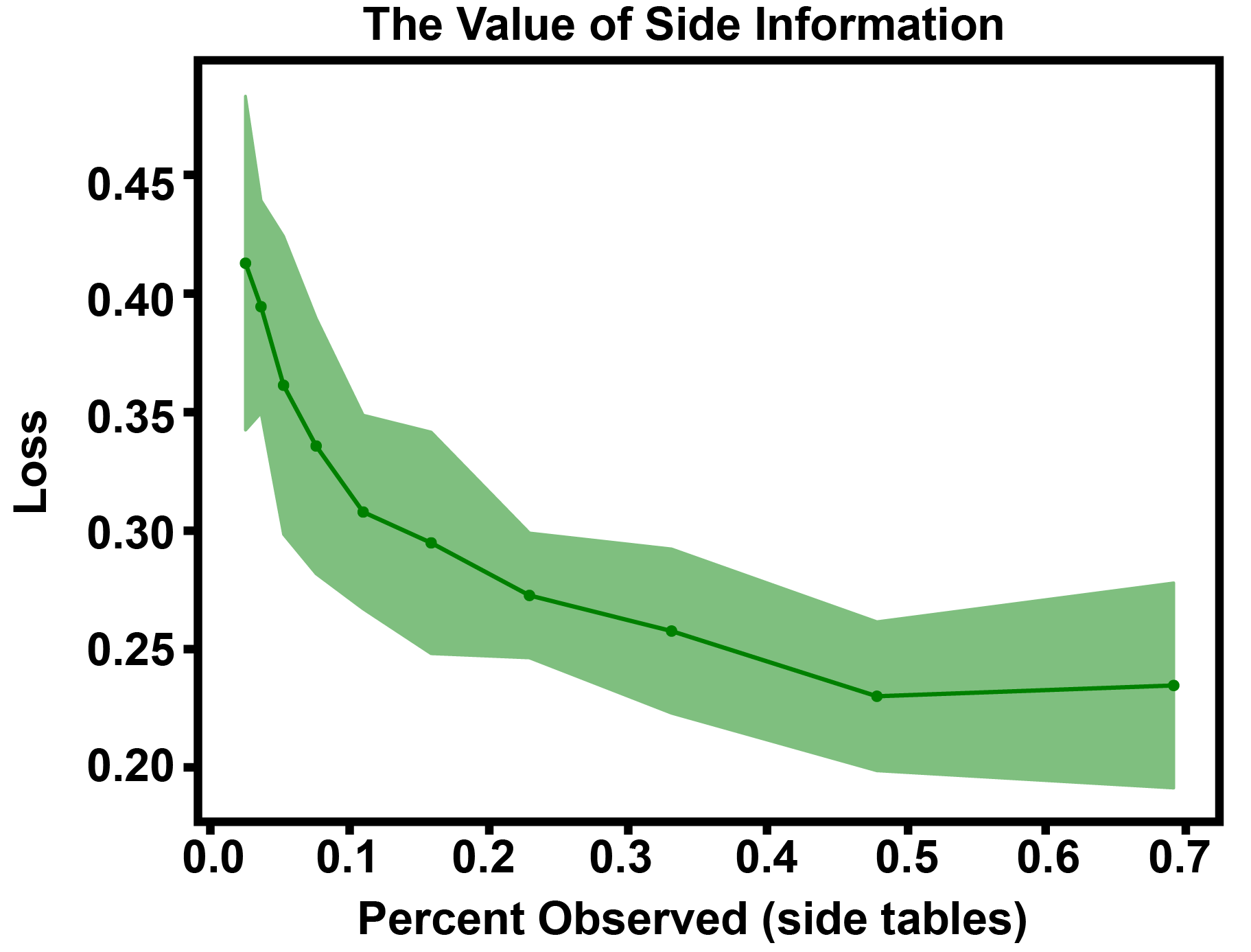}
\caption{Side information boost performance}\label{fig:side_info}
\end{wrapfigure}
\subsection{The Value of Side Information}
Do we gain anything by using the ``entire'' database for predicting missing entries of a particular table, compared to simply using one target tensor $\X^{\{1,2\}}$ ($\kw{student-course}$ table) for both training and testing? 
To answer this question, we fix the sparsity level of the \kw{student-course} table at 10\%, and train models with increasing levels of sparsity for other tensors $\X^{\{1,3\}}$ and $\X^{\{2,3\}}$ in the range $[0.025, 0.7]$. \cref{fig:side_info} shows that the side information in the form of $\kw{student-professor}$ and $\kw{course-professor}$ tables can significantly improve the prediction of missing records in the $\kw{student-course}$ table.

}{}

\end{document}